\documentclass[sigconf,nonacm]{acmart} % or [manuscript,nonacm] for 1-col

\renewcommand\shortauthors{} % also clears the author short header, just in case

\acmDOI{}
\acmISBN{}
\acmBooktitle{}            % belt-and-suspenders: clear booktitle

\setcopyright{none}

\makeatletter
\renewcommand\footnotetextcopyrightpermission[1]{}
\makeatother

\acmConference[]{}{}{}

\usepackage[utf8]{inputenc} % allow utf-8 input
\usepackage[T1]{fontenc}    % use 8-bit T1 fonts
\usepackage{hyperref}       % hyperlinks
\usepackage{url}            % simple URL typesetting
\usepackage{booktabs}       % professional-quality tables
\usepackage{amsfonts}       % blackboard math symbols
\usepackage{nicefrac}       % compact symbols for 1/2, etc.
\usepackage{microtype}      % microtypography
\usepackage{xcolor}         % colors

\usepackage{booktabs, pifont}
\usepackage{siunitx}
\usepackage{thmtools,thm-restate}
\usepackage{times}
\usepackage{soul}
\usepackage{graphicx}
\usepackage{amsmath,mathtools}
\usepackage{bm}
\usepackage{amsthm}
\usepackage{algorithm}
\usepackage{algpseudocode}

\let\oldComment=\Comment
\newcommand{\cmark}{\ding{51}} % checkmark
\newcommand{\xmark}{\ding{55}} % x mark

\renewcommand{\Comment}[1]{\oldComment{\texttt{#1}}}
\algnewcommand{\LeftComment}[1]{\Statex $\triangleright$ \texttt{#1}}
\algnewcommand{\RightComment}[1]{\Statex \leavevmode\hfill$\triangleright$ \texttt{#1}}
\algnewcommand\algorithmicinput{\textbf{Input:}}
\algnewcommand\Input{\item[\algorithmicinput]}%
\algnewcommand\algorithmicoutput{\textbf{Output:}}
\algnewcommand\Output{\item[\algorithmicoutput]}%
\algnewcommand\algorithmicinitial{\textbf{Initialize:}}
\algnewcommand\Initial{\item[\algorithmicinitial]}%

\newtheorem{theorem}{Theorem}

\newtheorem{proposition}{Proposition}
\newcommand{\FTRL}{\texttt{FTRL}}
\newcommand{\CRAC}{\texttt{CascadeRAC}}
\newcommand{\OLTR}{\texttt{OLTR}}

\newcommand{\Reg}{\textit{Reg}}

\newcommand{\CBARBAR}{\texttt{CBARBAR}}
\newcommand{\CBARBARNL}{\texttt{CascadeCBARBAR}}

\newcommand{\MSUCB}{\texttt{M\textsuperscript{2}UCB-V}}
\newcommand{\MUCB}{\texttt{MUCB-V}}
\newcommand{\CUCBV}{\texttt{CascadeUCB-V}}

\newcommand*{\inlineequation}[2][]{%
  \begingroup
    \refstepcounter{equation}%
    \ifx\\#1\\%
    \else
      \label{#1}%
    \fi
    \relpenalty=10000 %
    \binoppenalty=10000 %
    \ensuremath{%
      #2%
    }%
    ~\@eqnnum
  \endgroup
}
\usepackage{subcaption}
\usepackage{bbm}

\hypersetup{
	colorlinks   = true, %Colours links instead of ugly boxes
	urlcolor     = blue, %Colour for external hyperlinks
	linkcolor    = blue, %Colour of internal links
	citecolor   = blue %Colour of citations
}

\usepackage[colorinlistoftodos,prependcaption,textsize=tiny,textwidth=1cm,color=green]{todonotes}

\newcommand{\compilehidecomments}{false}%HIDE comments
\ifthenelse{ \equal{\compilehidecomments}{true} }{%
    \newcommand{\xutong}[1]{}
    \newcommand{\xw}[1]{}
    \newcommand{\mo}[1]{}
    \newcommand{\fg}[1]{}
    \newcommand{\rev}[1]{}
}
{
    \newcommand{\xutong}[1]{{\color{orange} [\text{Xutong:} #1]}}
    \newcommand{\xw}[1]{{\color{violet} [\text{XW:} #1]}}
    \newcommand{\fg}[1]{{\color{teal} [\text{FG:} #1]}}
    \newcommand{\mo}[1]{{\color{blue} [\text{MH:} #1]}}
    \newcommand{\rev}[1]{{\color{red}[#1]}}
}

\usepackage{xspace}

\usepackage{letltxmacro}
\LetLtxMacro{\originaleqref}{\eqref}
\renewcommand{\eqref}{Eq.~\originaleqref}

\begin{document}

\title[Online Learning to Rank under Corruption]{Online Learning to Rank under Corruption: \\A Robust Cascading Bandits Approach}


\author{Fatemeh Ghaffari}
\authornotemark[1]
\email{fghaffari@umass.edu}
\affiliation{%
  \institution{UMass Amherst}
  \city{Amherst}
  \state{Massachusetts}
  \country{USA}
}
\author{Siddarth Sitaraman}
\authornotemark[2]
\email{siddarth_sitaraman@brown.edu}
\affiliation{%
  \institution{Brown University}
  \city{Providence}
  \state{Rhode Island}
  \country{USA}
}
\author{Xutong Liu}
\authornotemark[3]
\email{xutongl@uw.edu}
\affiliation{%
  \institution{University of Washington - Tacoma}
  \city{Tacoma}
  \state{Washington}
  \country{USA}
}

\author{Xuchuang Wang}
\authornotemark[1]
\email{xuchuangwang@cs.umass.edu}
\affiliation{%
  \institution{UMass Amherst}
  \city{Amherst}
  \state{Massachusetts}
  \country{USA}
}

\author{Mohammad Hajiesmaili}
\authornotemark[1]
\email{hajiesmaili@cs.umass.edu}
\affiliation{%
  \institution{UMass Amherst}
  \city{Amherst}
  \state{Massachusetts}
  \country{USA}
}
% \affiliation{%

% \affiliation{%

% \affiliation{%

% \affiliation{%


\begin{abstract}

Online learning to rank (\OLTR{}) studies how to recommend a short ranked list of items from a large pool and improves future rankings based on user clicks. 
This setting is commonly modeled as cascading bandits, where the objective is to maximize the likelihood that the user clicks on at least one of the presented items across as many timesteps as possible.
However, such systems are vulnerable to click fraud and other manipulations (i.e., corruption), where bots or paid click farms inject corrupted feedback that misleads the learning process and degrades user experience.
In this paper, we propose \MSUCB{}, a robust algorithm that incorporates a novel mean-of-medians estimator, which to our knowledge is applied to bandits with corruption setting for the first time.
This estimator behaves like a standard mean in the absence of corruption, so no cost is paid for robustness. 
Under corruption, the median step filters out outliers and corrupted samples, keeping the estimate close to its true value. 
Updating this estimate at every round further accelerates empirical convergence in experiments.
Hence, \MSUCB{} achieves optimal logarithmic regret in the absence of corruption and degrades gracefully under corruptions, with regret increasing only by an additive term tied to the total corruption. 
Comprehensive and extensive experiments on real-world datasets further demonstrate that our approach consistently outperforms prior methods while maintaining strong robustness. 
In particular, it achieves a \(97.35\%\) and a \(91.60\%\) regret improvement over two state-of-the-art methods.
\end{abstract}

% <concept_id>10003752.10010070</concept_id>
% <concept_desc>Theory of computation~Theory and algorithms for application domains</concept_desc>
% <concept_id>10003752.10010070.10010071.10010079</concept_id>
% <concept_desc>Theory of computation~Online learning theory</concept_desc>
% <concept_id>10003752.10010070.10010071.10010261.10010276</concept_id>
% <concept_desc>Theory of computation~Adversarial learning</concept_desc>
% <concept_significance>300</concept_significance>
% <concept_id>10003752.10010070.10010071.10011194</concept_id>
% <concept_desc>Theory of computation~Regret bounds</concept_desc>

% \ccsdesc[500]{Theory of computation~Theory and algorithms for application domains}
% \ccsdesc[500]{Theory of computation~Regret bounds}
% \ccsdesc[500]{Theory of computation~Online learning theory}
% \ccsdesc[300]{Theory of computation~Adversarial learning}



\maketitle

\section{Introduction}
\label{sec:intro}

Learning to rank lies at the core of modern recommendation and information retrieval systems, where the goal is to present users with an ordered list of items tailored to their preferences \cite{cao2007learning, sculley2009large}.
Online learning to rank extends this paradigm by updating the recommendations as new feedback arrives in sequence, enabling systems to adapt quickly to user behavior~\cite{kveton2015cascading, combes2015learning, Katariya2016DCM, zoghi2017online, li2019online, li2020cascading, zuo2023adversarial}. 
For example on Yelp, the system ranks local businesses so that users find good restaurants and services quickly \cite{xie2025cascading}. 
In e-commerce, it orders products to maximize purchase probability \cite{hu2018reinforcement, huzhang2021aliexpress}. On music and movie platforms, it surfaces songs and films that match a user's taste while exploring new options \cite{zuo2023adversarial, vial2022minimax, li2020cascading}.
% Users typically review suggestions in order and click the first attractive item. 
% This behavior is modeled by the \textit{cascade model} \cite{}. 
% In the cascade model, each item has a click probability independent of the others. The objective is to maximize the probability that the user clicks any item on the list, which provably reduces to selecting the items with the highest click probabilities.

A widely studied model of user interactions in these applications is the \textit{cascade model} \cite{kveton2015cascading, vial2022minimax, zhong2021thompson, li2020cascading, zoghi2017online}, where users examine results sequentially from top to bottom and click the first attractive item.
The cascade model captures the position-dependent nature of user feedback and leads naturally to the cascading bandits framework, a special case of combinatorial bandits~\cite{kveton2015tight, wang2017improving, chen2016combinatorial, combes2015combinatorial} with non-linear rewards. Each item $k$ has an unknown click probability $\mu_k$, and the learner aims to sequentially construct ranked lists to maximize the probability of obtaining a click. The learning agent explores by estimating item attractiveness from past observations, while balancing the need to exploit high-probability items to maximize user satisfaction.

% Cascading bandits is one of the most widely used frameworks to model the cascade model in \OLTR{} \cite{}. 
% In this setting, which is a special case of general combinatorial bandits with non-linear rewards, each item \(k\) has its own distinct click probability \(\mu_k \in [0, 1]\). 
% These probabilities are unknown to the learner and must be learned over time. 
% The agent explores by using past observations to construct better-ranked lists and aims to maximize the overall reward. 
% The reward of a list is one if at least one item in the list receives a click, meaning it is enough for the user to click on any single attractive item. 
% Each arm’s reward is sampled from a distribution with expectation equal to its click probability.

Despite its success, practical \OLTR{} systems face substantial challenges from corrupted environments~\cite{Golrezaei2021Learning, xie2025cascading, zuo2023adversarial, immorlica2005click}. For example, click fraud in online advertising generates misleading clicks through bots, degrading both system revenue and user experience. Similarly, e-commerce platforms suffer from click farms or fake reviews that manipulate product rankings for profit. These corruptions, whether adversarial or stochastic, are pervasive in deployed systems. Thus, robust algorithms must achieve near-optimal performance in benign settings while degrading gracefully under corruption, even when the amount of corruption is unknown.

% Practical \OLTR{} systems must be able to operate in corrupted environments, as they face many forms of attacks and corruptions regularly. 
% One common example is click fraud, where bots repeatedly click on ads without real interest, lowering their click-through rate and giving an unfair advantage to competitors \cite{}. 
% Another example comes from click farms in e-commerce, where fake accounts generate artificial clicks or reviews to promote certain products and bury others \cite{}. 
% To handle these challenges, we need robust algorithms that achieve optimal performance in trustworthy settings, yet degrade gracefully under corruption, even without prior knowledge of how much corruption is present \cite{}.

Designing a robust algorithm for the cascading bandit setting poses unique challenges. The non-linear reward function breaks the applicability of standard linear confidence bounds widely used in existing works, making them face significant limitations: some achieve optimality in uncorrupted settings but fail under corruption \cite{kveton2015cascading, vial2022minimax}; others offer robustness but sacrifice optimal guarantees \cite{Golrezaei2021Learning, xie2025cascading}; yet others rely on epoch-based designs that slow convergence and underperform in practice \cite{xu2021simple}. As a result, there remains a critical gap in the literature: we lack algorithms that are simultaneously (i) optimal in trustworthy environments, (ii) provably robust under adversarial corruption, and (iii) reliable in real-world deployments.

\noindent\textbf{Contributions.}
This paper aims to address the above challenges of corruption-robust cascading algorithms by addressing the following central question:
\emph{How can we efficiently and robustly learn in corrupted cascading bandit environments in a way that also performs reliably on real-world data?}

\paragraph{Algorithm design}
We develop Model selection calibrated Mean-of-medians Variance-aware UCB (\MSUCB{}), an algorithm that remains robust under adversarial corruption, yet performs optimally when no corruption is present, without requiring prior knowledge of the corruption level. The key design ideas behind \MSUCB{} include: (i) incorporating a calibrated mean-of-medians estimator that ensures reliable performance both with and without corruption, (ii) introducing a variance-aware refinement of  \texttt{UCB} to further tighten the regret bound, and (iii) employing a model selection mechanism to automatically adapt to unknown corruption levels. The full algorithmic details are presented in Section~\ref{sec:MSUCB}.

First, we incorporate a calibrated mean-of-medians mechanism that leverages a robust median-based estimator. We show that in the absence of corruption, the mean-of-medians behaves like a \texttt{UCB} mean estimator and achieves optimal performance. 
Under corruption, the median step selects uncorrupted samples with high probability, ensuring robustness. 
While this estimator has been previously used in heavy-tailed bandits~\cite{zhong2021breaking, xue2023efficient}, applying it to adversarial corruption is novel and requires addressing new challenges in careful adaptation to guarantee sufficient uncorrupted samples per item.
Because click feedback is asymmetric and typically Bernoulli, unlike the symmetric assumptions common in heavy-tailed bandits, we introduce a calibration step that maps the mean-of-medians back to the underlying Bernoulli mean, producing estimates centered around the desired true value.
Then, by combining this estimator with the state-of-the-art variance-aware \texttt{UCB} radius, we develop an algorithm that is robust against corruption, albeit with the requirement of prior knowledge of the corruption level.
Last, and to remove the dependence on prior knowledge of corruption level, we incorporate the model selection framework inspired by ~\citet{wei2022model} and develop the corruption-agnostic Model Selection \MUCB{} (\MSUCB{}).
Because the algorithm updates its estimates every round, it adapts quickly to changes and rapidly converges to the true optimal items once corruption is removed.

% Our second algorithm, \CBARBARNL{}, generalizes the \texttt{CBARBAR} algorithm of~\citet{xu2021simple} to the nonlinear \CMAB{} setting by replacing the linear reward assumption with a general nonlinear reward function.  
\paragraph{Theoretical guarantees.}
In Section~\ref{sec:MSUCB_th}, we provide a regret analysis establishing near-optimal bounds both in the corruption-free regime and in terms of dependence on the total corruption.
 In particular, we establish that the regret of \MSUCB{} is bounded by  
\(
O \left( K C  + \frac{K\log T}{\Delta} \right),
\) where \(T\) is the time horizon, \(C\) is the total corruption level, \(K\) is the number of items, \(S^*\) is the optimal recommended list, and \(\Delta\) is the minimum gap between any suboptimal item \(k\) and the worst optimal item. 
\textit{This result means that in the absence of corruption, \MSUCB{} achieves optimal regret, and under corruption, the additional term is optimal up to a multiplicative factor of \(K\).
}
The main challenge for the regret analysis is the bias of the mean-of-medians estimator. For Bernoulli data, a block median estimates \(q_b(\mu)\), the probability that a size-\(b\) block has a majority of ones, rather than the true mean \(\mu\). 
We therefore calibrate the estimator by inverting \(q_b\) to map the mean-of-medians back to \(\mu\). 
To show that the calibrated estimator is centered at the true mean, we first derive a Bernstein-type concentration bound for the average of the block-median indicators, then use the mean value theorem with the inverse map \(g_b=q_b^{-1}\) to transfer this deviation to the mean parameter while controlling the local slope via \(1/q_b'(\xi)\). 
Combined with the 1-Lipschitz, self-bounding property of the Bernoulli variance \(v(p)=p(1-p)\), this yields a variance-adaptive and corruption-robust cascading bandit algorithm. 
In the absence of corruption, \MSUCB{} attains the optimal gap-dependent regret for cascading bandits, and to our knowledge, it is the first corruption-robust method in this setting to match the stochastic optimal regret.
The corruption term in our regret is suboptimal by a factor \(K\) relative to the \(O(C)\) benchmark. 
As a side contribution, we adapt the corruption-robust \CBARBAR{} algorithm of \citet{xu2021simple} from combinatorial bandits with linear rewards to the cascading bandit setting. 
We call the resulting method \CBARBARNL{}. 
Appendix~\ref{sec:appendix_C} presents the algorithm, its regret bound, and a proof sketch. 
In the presence of corruption, \CBARBARNL{} is suboptimal by a factor \(d\), the list length.
Nevertheless, as shown in Section~\ref{sec:NumRes}, \MSUCB{} consistently outperforms \CBARBARNL{} even under heavy corruption, largely because it avoids the doubling-epoch schedule that delays updates and slows convergence.

% To obtain this bound, we leverage the Lipschitz continuity of the cascading bandit model, which ensures that after \(O(dC)\) rounds the median step picks uncorrupted samples with high probability. 
% The key technical difficulty lies in the nonlinearity introduced by the median operation, which we overcome using Hoeffding’s inequality and Chernoff tail bounds. 
% As a consequence, we prove that after \(O(dC)\) rounds the mean-of-medians estimator converges to the true uncorrupted mean of each arm.

\paragraph{Empirical evaluation.} 
Last, in Section~\ref{sec:NumRes}, we validate our approach through extensive experiments on real-world datasets, demonstrating strong performance across a wide range of corruption levels, including the stochastic setting. We conduct experiments on three large-scale real-world datasets: Yelp \cite{yelp2024open}, MovieLens~\cite{Harper2015MovieLens}, and LastFM~\cite{Schedl2016LFM}, which together represent some of the most common applications of \OLTR{}.
Our empirical results support the theoretical guarantees across several scenarios.  
We validate our algorithms on both synthetic and real-world datasets, comparing against strong baselines including \CUCBV{}
\cite{vial2022minimax}, \FTRL{} \cite{Ito2021Advances}, \CRAC{} \cite{xie2025cascading}, and \CBARBARNL{}.
In a set of representative experiments, our method's cumulative regret after \(40k\) rounds improves \FTRL{} by \(99.60\%\), \CUCBV{} by \(97.35\%\), \CRAC{} by \(91.60\), and \CBARBARNL{} by \(98.41\%\).
% The first is \CUCBgypsum-gpu001.unity.rc.umass.edu{}, which is optimal in uncorrupted settings but not robust \cite{wang2017improving}. 
% The second is \FTRL{}, a best-of-both-worlds algorithm that is optimal in both stochastic and adversarial regimes \cite{Ito2021Advances}. 
% We also include \CRAC{}, which is robust to corruption but not theoretically optimal \cite{xie2025cascading}. 
% Finally, we evaluate against \CBARBARNL{}, an extension of the \CBARBAR{} algorithm introduced by \citet{xu2021simple} to the cascading bandit setting, an algorithm that achieves regret bounds close to optimal in theory, but it fails to perform optimally in practice.

\subsection{Related Work}

\begin{table*}[t]
\centering
\caption{Comparison of cascading bandit algorithms in stochastic and corrupted settings.}
\label{tab:algo_comparison}
\setlength{\tabcolsep}{6pt}
\small
\begin{tabular}{@{}l l c l c c@{}}
\toprule
\textbf{Algorithm} 
& \textbf{Regret w/o corruption} 
& \textbf{Stoch.\ LB}\textsuperscript{$\dagger$} 
& \textbf{Regret w/ corruption} 
& \textbf{Robust} 
& \textbf{Corr.\ factor} \\
\midrule
\texttt{CascadeUCB-V}~\cite{vial2022minimax}
& \(\textstyle O  \left( \sum_{k \notin S^*} \tfrac{\log T}{\Delta_{k}} \right)\)
& \cmark 
& -- 
& \xmark 
& -- \\
\addlinespace[2pt]
\texttt{CascadeRAC}~\cite{xie2025cascading} 
& \(\textstyle O  \left(\sum_{k=d+1}^{K}  \tfrac{d  \log T \log(KT)}{\Delta_k}\right)\)
& \xmark 
& \(\textstyle O  \left(\sum_{k=d+1}^{K}  \tfrac{d \left(CK \log(KT) + \log T\right)\log(KT)}{\Delta_k}\right)\)
& \cmark 
& Multiplicative \\
\texttt{FORC}~\cite{Golrezaei2021Learning}
& \(\textstyle O  \left(\sum_{k=1}^{K} \sum_{j = k+1}^{K} \tfrac{ \log T \log(KT)}{\Delta_{kj}}\right)\)
& \xmark 
& \(\textstyle O  \left(\sum_{k=d+1}^{K}  \tfrac{d \left(CK \log(KT) + \log T\right)\log(KT)}{\Delta_k}\right)\)
& \cmark 
& Multiplicative \\
\addlinespace[2pt]
\CBARBARNL{}~\cite{xu2021simple} 
& \(\textstyle O  \left(  \tfrac{d^{2} K}{\Delta_{\min}} \log^{2} T \right)\)
& \xmark 
& \(\textstyle O  \left(dC +  \tfrac{d^{2} K}{\Delta_{\min}} \log^{2} T \right)\)
& \cmark 
& Additive \\
\addlinespace[2pt]
\MSUCB{}~(Ours)
& \(\textstyle O  \left( \tfrac{K\log T}{\Delta} \right)\)
& \cmark 
& \(\textstyle O  \left( KC + \tfrac{K\log T}{\Delta}\right)\)
& \cmark 
& Additive \\
\bottomrule
\end{tabular}
\normalsize
\vspace{4pt}
\begin{minipage}{0.95\textwidth}
\footnotesize \textit{Note} (\textsuperscript{$\dagger$}): \cmark\; indicates the algorithm matches the known stochastic lower bound in the uncorrupted stochastic setting, \xmark\; indicates it does not.
\end{minipage}
\end{table*}

\paragraph{Cascading Bandits}

\citet{kveton2015cascading} introduced the cascading bandit formulation along with \texttt{UCB}-style algorithms (\texttt{CascadeUCB} and \texttt{CascadeKL-UCB}) that achieve regret guarantees, establishing the now-standard “examine-until-first-click” feedback model. 
\citet{combes2015learning} developed gap-dependent regret lower bounds and proposed efficient learning-to-rank algorithms.
\citet{vial2022minimax} establish gap-independent lower bounds for cascading bandits and introduce \CUCBV{}, a variance-adaptive  \texttt{UCB} algorithm that matches those bounds. \CUCBV{}'s gap-dependent regret is the optimal \(O \left(\sum_{k\notin S^*}\frac{\log T}{\Delta_k}\right)\), as shown in Table~\ref{tab:algo_comparison}.
\citet{Katariya2016DCM} generalized cascading bandits to the dependent click model, extending cascades to handle multiple clicks. 
\citet{Largee2016Multiple} introduced the position-based model, which explicitly accounts for position bias and provides both lower bounds and efficient algorithms in multi-play settings. 
\citet{zoghi2017online} studied learning under a broad class of stochastic click models, including both the cascade and position-based models, and established gap-dependent regret bounds with strong empirical performance. 
Lattimore et al.\ (2020) further provided gap-independent regret bounds for generalized click models. 
\citet{zhong2021thompson} introduced Thompson Sampling variants for cascading bandits, offering Bayesian learning methods with regret guarantees and competitive empirical performance.
Another line of work considers contextual cascading bandits, which incorporate item features and position discounts to better capture practical scenarios \cite{li2016contextual, zong2016cascading, li2018online, li2019online, li2020cascading}. Finally, cascading bandits can be viewed as a special case of combinatorial bandits with non-linear reward functions~\cite{li2016contextual,chen2016combinatorial,wang2017improving,liu2022batch}, linking this setting to the broader combinatorial bandit literature.
To facilitate learning in such settings, structural conditions like Lipschitz continuity~\citep{wang2017improving,wang2024stochastic}, and more recently, inverse Lipschitz continuity~\citep{chen2025continuous} have been adopted to ensure that changes in base-arm rewards translate to bounded and predictable changes in the overall reward. 
Most prior work on cascading bandits—and even on broader combinatorial bandits—ignores adversarial corruption and thus lacks robustness guarantees. In contrast, \MSUCB{} enjoys optimal regret in the uncorrupted stochastic regime and remains robust under corruption.

\paragraph{Bandits with Adversarial Corruption}
The study of adversarial corruption in bandits was initiated by \citet{lykouris2018stochastic}, who introduced a corruption model for the classical stochastic MAB and showed that any algorithm with logarithmic stochastic regret must incur regret linear w.r.t corruption level \(C\). 
Building on this, \citet{gupta2019better} proposed the corruption-agnostic algorithm \texttt{BARBAR}, which achieves a regret bound of \(O \big(KC + K \log T / \Delta_{\min}\big)\), thereby replacing the earlier \emph{multiplicative} dependence on \(C\) with an \emph{additive} term and doing so without requiring prior knowledge of the corruption budget.
However, there has been limited work on cascading bandits with corruption. 
\citet{Golrezaei2021Learning} introduced two algorithms, \texttt{FAR} and \texttt{FORC}, for ranking with fake users. 
\texttt{FAR} assumes knowledge of the fakeness level \(C\) and inflates confidence intervals accordingly. 
To propose a corruption-agnostic method, they design \texttt{FORC}, which runs multiple layers with different sampling frequencies and implicit corruption guesses, in the spirit of the corruption agnostic layering of \citet{lykouris2018stochastic}. 
In this setting with up to \(C\) fake users over a horizon \(T\), \texttt{FORC} achieves the regret
\[
O \left(\bigl(K^2 C + \log T\bigr)\sum_{k=1}^{K}\sum_{j=k+1}^{K}\frac{\log(KT)}{\Delta_{kj}}\right).
\]
As shown in Table~\ref{tab:algo_comparison}, this algorithm does not match the stochastic lower bound in the uncorrupted setting and, under corruption, suffers from a multiplicative dependence on \(C\).
Later, \citet{xie2025cascading} proposed \texttt{CascadeRAC}, which extends active–arm elimination, similar to the method of \citet{lykouris2018stochastic}, to cascading bandits. They obtain the regret
\[
O \left(\sum_{k=d+1}^{K} \frac{d \big(CK \log(KT) + \log T\big)\,\log(KT)}{\Delta_k}\right).
\]
This improves \texttt{FORC}'s dependence on \(K\), but remains suboptimal in the stochastic stochastic regime and still exhibits a multiplicative dependence on the corruption budget, as summarized in Table~\ref{tab:algo_comparison}.

In a related line of work, \citet{xu2021simple} extended \texttt{BARBAR} to combinatorial bandits with linear rewards, proposing \CBARBAR{} with regret \(\tilde{O} \left(C + \tfrac{d^{2}K}{\Delta_{\min}}\right)\). We extend this approach to the cascading setting (see Appendix~\ref{sec:appendix_C}), yielding \CBARBARNL{}, and show that it achieves the regret bound
\[
O \left(dC \;+\; \frac{d^{2} K}{\Delta_{\min}} \log^{2} T \right).
\]
While \CBARBARNL{} attains an \emph{additive} dependence on the corruption level \(C\), it remains suboptimal in the stochastic regime (Table~\ref{tab:algo_comparison}).
Our algorithm, \MSUCB{}, enjoys optimal regret in the uncorrupted setting and only an additive dependence on corruption, yielding total regret
\[
O\left(KC + \tfrac{K\log T}{\Delta}\right).
\]
To the best of our knowledge, it is the first corruption-robust cascading bandit method that matches the stochastic optimum. Although \CBARBARNL{} theoretically has a slightly smaller corruption factor, Section~\ref{sec:NumRes} shows that \MSUCB{} outperforms it even under heavy corruption.

Last, a parallel line of work studies adversarial combinatorial bandits~\cite{han2021adversarial} and best-of-both-worlds (BoBW) combinatorial bandits~\cite{Ito2021Advances}. 
However, there is no known Adversarial or BoBW work that handles combinatorial bandits with general nonlinear function such as cascading bandits.
Our experiments show that the BoBW FTRL method of \citet{Ito2021Advances} underperforms corruption-robust baselines in corrupted stochastic environments.

\section{Problem Setting}
\label{sec:model}
\subsection{Cascading Bandits}
We consider a cascading bandit with \(K\in\mathbb{N}^+\) items, indexed by \([K]\coloneq\{1,\dots,K\}\). 
Each item \(k\) is associated with an unknown click probability \(\mu_k\in[0,1]\), collected in the vector \(\bm{\mu}\coloneq(\mu_1,\dots,\mu_K)\in[0,1]^K\). 
When item \(k\) is examined at round \(t\), the click outcome \(X_k(t) \in \{0,1\}\) follows a Bernoulli distribution, \(X_k(t) \sim \mathrm{Ber}(\mu_k)\).  
We denote \(\bm{X}(t) \coloneq (X_1(t), \dots, X_K(t))\) and following prior cascading bandit works~\cite{kveton2015cascading,vial2022minimax}, we assume that outcomes are independent across items.

The interaction proceeds over \(T \in \mathbb{N}^+\) decision rounds.  
At each round \(t \in [T]\), the learner recommends an ordered list of \(d \in \mathbb{N}^+\) items,
$
S(t) \coloneq (s_1(t), \dots, s_d(t)) \in [K]^d.
$
The user examines the list sequentially from the top and clicks the first attractive item, after which they stop browsing.  
We define the user’s stopping position as
$
\kappa_t \coloneq \min \{ i \in [d] : X_{s_i(t)}(t) = 1 \},
\quad \text{and set } \kappa_t = d + 1 \text{ if no item is clicked.}
$
The learner observes the feedback \(X_{s_i(t)}(t)\) for all positions \(i \le \kappa_t\) and receives no feedback for unexamined positions \(i > \kappa_t\).

% The user reviews the list from the top to the bottom. 
% They click the first attractive item and then stop, leaving the rest of the items unexamined.
% Define the user's stopping position by
% \[
% \kappa_t \coloneq \min\bigl\{ i \in [d] : X_{s_i(t)}(t)=1 \bigr\},
% \quad\text{with }\kappa_t \coloneq d+1 \text{ if no clicks.}
% \]
% The learner observes the feedback \(X_{s_i(t)}(t)\) for all positions \(i \le \kappa_t\) and receives no feedback for positions \(i > \kappa_t\).

% The reward is one only if the user clicks on an item on the list, and is otherwise zero, formally defined as
% \begin{align}\label{eq:reward}
% R(S(t), \bm{X_t}) \coloneq 1 - \prod_{k \in S(t)} (1 - X_{k}(t)).
% \end{align}
The reward at round $t$ equals one if the user clicks at least one item, and zero otherwise, which is formally defined as:
\begin{align}\label{eq:reward}
R(S(t), \bm{X}(t)) \coloneq 1 - \prod_{k \in S(t)} (1 - X_k(t)).
\end{align}
Since \(X_{k}(t)\)'s are independent across $k$, we have \(\mathbb{E}\left[R(S(t), \bm{X}(t))\right] = R(S(t), \bm{\mu})\).
The learner's performance is evaluated by the expected cumulative regret defined under the true reward mean before corruption as
\[
\Reg(T) = \mathbb{E} \left[\sum_{t=1}^{T} \left( R(S^*, X_t) - R(S(t), X_t) \right)\right],
\]
where \( S^* = \arg\max_{S \in [K]^d}R(S, \bm{\mu}) \) is the optimal list selected with prior knowledge of \(\bm{\mu}\), and \(S(t)\) is the item list chosen by the learner at round \( t \).

\subsection{Adversarial Corruption Model}
We further consider a corruption setting in which an adversary can arbitrarily corrupt the item click rewards.  
Let \( {c}_{k, t} \in \mathbb{R} \) denote the corruption applied to item \(k\) at round \( t \), and define the cumulative corruption budget, unknown to the learner, as 
% \inlineequation[eq:c]{C \coloneq  \max_{Z \in \mathcal{Z}} \sum_{t = 1}^T \left| r(\bm{r}_t + \bm{c}_t, Z) - r(\bm{r}_t, Z) \right|}. 
\begin{align}\label{eq:c}
\sum_{t = 1}^T \max_{k \in S(t)} |c_{k, t}| \le C.
\end{align}
The learner observes corrupted feedback but is evaluated with respect to the underlying stochastic environment. 
This corruption formulation follows \citet{xie2025cascading}, and is conceptually aligned with the “fake user” corruption model introduced by \citet{Golrezaei2021Learning}, where corruption corresponds to a bounded number of perturbed rounds.
% This definition of corruption has also been adopted by \citet{xie2025cascading}. 
% \citet{Golrezaei2021Learning}, use an intuitive notion of corruption defined as the number of rounds with fake users, which is also equivalent to ours when expressed formally. 
The objective is to develop robust algorithms whose expected regret scales gracefully with the cumulative corruption level \( C \).
The overall learning protocol with corruption is described as follows in Algorithm~\ref{alg:adversarial_corruption}:
\begin{algorithm}[ht]
\caption{Corruption procedure for cascading bandits}
\label{alg:adversarial_corruption}
\begin{algorithmic}[1]
\For{round \( t = 1, \dots, T \)}
    \State Learner recommends an ordered list \( S(t) \in [K]^d \)
    \State Adversary chooses a corruption vector \( \bm{c}(t) \in \mathbb{R}^d \)
    \State Rewards \( \bm{X}(t) \in [0,1]^K \) are independently drawn with mean \( \mu_k \) for each item \( k \)
    \State Learner observes corrupted reward \( R(S(t), \bm{X}(t) + \bm{c}(t)) \)
\EndFor
\end{algorithmic}
\end{algorithm}

\section{The \MSUCB{} Algorithm}
\label{sec:MSUCB}
% In this section, we present our proposed algorithm, \MSUCB{}, with algorithm details in Section~\ref{sec:MSUCB_alg} and its regret analysis in Section~\ref{sec:MSUCB_th}.

% \subsection{Algorithm Design}\label{sec:MSUCB_alg}

Do We first address the setting with a known corruption level by introducing a calibrated mean-of-medians estimator in Algorithm~\ref{alg:mean_of_medians_cal}, which provides robust estimates under corruptions. 
% and use it to derive a corruption-aware variant of the \texttt{CascadeUCB-V} algorithm~\cite{vial2022minimax}. 
% The resulting procedure is summarized in Algorithm~\ref{alg:mom_ucb_v}. 
Using this estimator, we derive a corruption-aware variant of the \texttt{CascadeUCB-V} algorithm~\cite{vial2022minimax}, presented in Algorithm~\ref{alg:mom_ucb_v}.
Building on this, we further incorporate the model-selection framework of \cite{wei2022model} to obtain the final corruption-agnostic algorithm, \MSUCB{}, which automatically adapts to both stochastic and corrupted environments without requiring prior knowledge of the corruption level.

% \begin{algorithm}[ht]
% \caption{Median-UCB (\MUCB{})}
% \label{alg:cucb_median}
% \begin{algorithmic}[1]
% \Input Corruption level \( C \), oracle access, horizon \( T \)

% \For{each item \( k \in [K] \)}
%     \State \( \bm{X_k} \coloneq [] \) \Comment{Initialize list of observed rewards} \label{line:init_sk}
% \EndFor

% \For{rounds \( t = 1, \dots, \lceil 10CK / d \rceil \)}
%   \State \( s \gets ((t - 1) d \bmod K) \)
%   \State \(  S_t \gets \{  ((s{+}r)\bmod K)+1  :  r=0,\dots,d - 1  \} \)
%   \State \textbf{play} \( S_t\) \Comment{t-th contiguous size-\(d\) block, cycled over \(K\) items; each item appears \(\approx 10C\) times}
% \EndFor

% \For{rounds \( t = \lceil 10CK / d \rceil, \dots, T \)}
%     \For{each item \( k \in [K] \)}
%         \State \( \hat{\mu}_k \gets \texttt{MeanOfMedians}(\bm{X_k}) \) 
%        \State \( \rho_k \gets \sqrt{\frac{3 \log t}{2 n_{k,t}}} + \mathbbm{1}\{|\bm{X}_k| < 10C\} \) \Comment{Extended radius when few samples} \label{line:median_ucb_conf_radius}
%         \State \( T_{k,t-1} \gets |\bm{X_k}| \)
%         \State \( \bar{\mu}_k \gets \min\left\{ \hat{\mu}_k + \rho_k, 1 \right\} \) \Comment{Upper confidence bound} \label{line:ucb_median}
%     \EndFor

%     \State \(  S_t \coloneq \max_{d}\{\bar{\mu}_1, \dots, \bar{\mu}_K\} \) \Comment{Select top d items.} \label{line:oracle_median}
%     \State Play \(  S_t \), for  each item \( k \in  S_t \), observe feedback \(X_{k, t}\) and append it to \( \bm{X_k} \)
% \EndFor
% \end{algorithmic}
% \end{algorithm}

\begin{algorithm}[ht]
\caption{\MUCB{} (calibrated Mean-of-medians variance-aware UCB)}
\label{alg:mom_ucb_v}
\begin{algorithmic}[1]
\Input Horizon \(T\), list size \(d\), corr. budget \(C\),
 Constants \(\alpha,A,B>0\)
% \For{each item \(k\in[K]\)}
%   \State \( \bm{X}_k \gets \emptyset [] \) \Comment{Observed clicks of item K} \label{line:init_xk}
% \EndFor

\For{each item \(k \in [K]\)}\label{line:warmup_start}
\State \( \bm{X}_k \gets [\emptyset] \) \Comment{Observed clicks of item K} \label{line:init_xk}
  \For{\(r=1\) to \(10C\)}\label{line:explore1}
    \State Recommend \(S_k  =  \{k,\, \text{any } d{-}1 \text{ items from } [K]\setminus\{k\}\}.\)\label{line:explore2}
    \State Observe click feedback and append to \(X_k\) for all observed \(k\)\label{line:explore3}
  \EndFor
\EndFor\label{line:warmup_end}
\For{rounds \( t = 10KC, \dots, T \)}
  \For{each item \(k\in[K]\)}
    \State \( T_{k}(t)\gets |\bm{X}_k| \),   \(b\gets \lceil \alpha\log \max\{T_{k}(t),2\}\rceil \) \label{line:nkG}
    \State \( \hat\mu_k \gets \texttt{CalibratedMeanOfMedians}(\bm{X}_k, b) \) \label{line:est_mu}
    \State \( \hat v_k \gets \hat\mu_k(1-\hat\mu_k) \) \Comment{empirical variance proxy} \label{line:var_proxy}
    \State \( s \gets \max\{1, T_{k}(t)\} \) \Comment{avoid divide-by-zero} \label{line:safe_s}
    \State \( \rho_k \gets A\sqrt{\tfrac{\hat v_k  \log t}{ s }} +  B \tfrac{\log t}{ s }\) \label{line:var_radius}
    \State \( \bar\mu_k \gets \min\{ \hat\mu_k+\rho_k, 1 \} \) \label{line:ucb_k}
  \EndFor
  \State \(  S_t \gets \text{Top-d items by } \bar\mu_k \) \label{line:select_topd}
  \State Play \( S_t\),Observe click feedback and append to \(X_k\) for all observed \(k\) \label{line:play_append}
\EndFor
\end{algorithmic}
\end{algorithm}

% \begin{algorithm}[ht]
% \caption{\texttt{MeanOfMedians}}
% \label{alg:mean_of_medians}
% \begin{algorithmic}[1]
% \Input Reward vector \( \bm{X} \), number of groups \( G = \lceil \log T \rceil \) \label{line:median_input}
% \If{\( |\bm{X}| < G \)} \label{line:check_samples}
%     \State \Return mean of \( \bm{X} \) \label{line:return_mean_few_samples}
% \EndIf
% \State Uniformly at random, partition \( \bm{X} \) into \( G \) groups of (approximately) equal size: \( \bm{X}^{(1)}, \dots,  \bm{X}^{(G)} \) \label{line:partition_groups}
% \For{each group \( j = 1, \dots, G \)} \label{line:for_each_group}
%     \State \( m_j \gets \texttt{Median}(\bm{X}^{(j)}) \) \label{line:median_of_group}
% \EndFor
% \State \Return \texttt{Mean}\((m_1, \dots, m_G)\) \label{line:return_mean_of_medians}
% \end{algorithmic}
% \end{algorithm}

\begin{algorithm}[ht]
\caption{\texttt{CalibratedMeanOfMedians}}
\label{alg:mean_of_medians_cal}
\begin{algorithmic}[1]
\Input Reward vector \( \bm{X}\), number of groups \( b=\lceil\alpha \log T \rceil \) \label{line:mm_input}
\If{\( |\bm{X}| < b \)} \label{line:mm_few}
    \State \Return \( \texttt{Mean}(\bm{X}) \)  \Comment{not enough samples} \label{line:mm_mean_fallback}
\EndIf
\State Set block size \( b \gets 2\big\lfloor \tfrac{|\bm{X}|}{2b}\big\rfloor + 1 \)  \Comment{nearest odd} \label{line:mm_block}
\State Uniformly partition \( \bm{X} \) into \( b \) blocks \( \bm{X}^{(1)},\dots,\bm{X}^{(G)} \) with sizes \( \approx b \) \label{line:mm_partition}
\For{each \( j=1,\dots, b \)} \label{line:mm_loop}
    \State \( M_j \gets \texttt{Median}\big(\bm{X}^{(j)}\big) \in \{0,1\} \) \label{line:mm_median}
\EndFor
\State \( \overline M \gets \texttt{Mean}(M_1,\dots,M_b) \) \label{line:mm_meanofmed}
\State \( \hat\mu \gets \texttt{Calibrate}(b,\overline M) \) \Comment{invert majority map} \label{line:mm_cal}
\State \Return \( \hat\mu \) \label{line:mm_ret_cal}
\end{algorithmic}
\end{algorithm}

\begin{algorithm}[ht]
\caption{\texttt{Calibrate} (inverse \(g_b=q_b^{-1}\) by bisection)}
\label{alg:calibrate}
\begin{algorithmic}[1]
\Input odd \( b\ge 3 \), target \( y\in[0,1] \), tolerance \( \eta>0 \), max iters \( N \) \label{line:cal_input}
\State Define \( q_b(p) \gets \sum_{r=(b+1)/2}^{b} \binom{b}{r}  p^{r}(1-p)^{b-r} \) \label{line:cal_q}
\State \( \ell \gets 0,\ r \gets 1 \) \label{line:cal_bracket}
\For{\( t=1,\dots,N \)} \label{line:cal_loop}
    \State \( m \gets (\ell+r)/2 \); \quad \( v \gets q_b(m) \) \label{line:cal_eval}
    \If{\( |v-y| \le \eta \)} \State \Return \( m \) \label{line:cal_tol}
    \ElsIf{\( v < y \)} \State \( \ell \gets m \) \label{line:cal_left}
    \Else \State \( r \gets m \) \label{line:cal_right}
    \EndIf
\EndFor
\State \Return \( (\ell+r)/2 \) \Comment{final bracket midpoint} \label{line:cal_return}
\end{algorithmic}
\end{algorithm}

\paragraph{The Mean-of-Medians Estimator}
We propose our robust calibrated mean-of-medians estimator, detailed in Algorithm~\ref{alg:mean_of_medians_cal}. At each round \(t\), the samples collected for item \(k\) are randomly divided into \(b = \lceil \alpha \log T_{k,t} \rceil\) groups (Line~\ref{line:mm_partition}). The median of each group is then computed, and their average is taken to form the mean-of-medians estimator (Line~\ref{line:mm_meanofmed}). Because each item's clickability follows an asymmetric Bernoulli distribution, the expected median of its samples generally deviates from the true mean.  To correct this bias, we calibrate the mean-of-medians estimator using the procedure described in Algorithm~\ref{alg:calibrate}. Specifically, we invert the monotone function \(y = q_b(x)\), which maps the true Bernoulli mean to the expected median of \(b\) samples. This inversion produces a calibrated estimate that closely recovers the underlying true mean corresponding to the observed mean-of-medians.

The \texttt{Calibrate} procedure numerically inverts the majority function \(q_b\) to obtain \(\hat{\mu} = g_b(y)\) for a given target \(y \in [0,1]\). 
It takes as input the odd block size \(b\), target value \(y\), tolerance \(\eta > 0\), and maximum number of iterations \(N\). %(Line~\ref{line:cal_input}) 
We explicitly define 
\[
q_b(p) = \sum_{r=(b+1)/2}^{b} \binom{b}{r} p^r (1-p)^{b-r},
\]
which gives the probability that a block of \(b\) Bernoulli samples with mean \(p\) has a majority of ones (Line~\ref{line:cal_q}). 
The algorithm initializes a search interval \([\ell, r] = [0,1]\) (Line~\ref{line:cal_bracket}) and performs up to \(N\) bisection steps (Line~\ref{line:cal_loop}). 
In each step, it computes the midpoint \(m = (\ell + r)/2\) and evaluates \(v = q_b(m)\) (Line~\ref{line:cal_eval}). 
If \(|v - y| \le \eta\), it returns \(m\) as the estimate (Line~\ref{line:cal_tol}). 
Otherwise, it updates the interval using the monotonicity of \(q_b\), setting \(\ell \leftarrow m\) if \(v < y\) (Line~\ref{line:cal_left}) and \(r \leftarrow m\) otherwise (Line~\ref{line:cal_right}). 
If the tolerance is not met after \(N\) iterations, the algorithm returns the final midpoint \((\ell + r)/2\) as the calibrated estimate (Line~\ref{line:cal_return}).

This method guarantees that the fraction of corrupted samples in any group remains below \(1/2\) with high probability, provided the total number of samples exceeds \(O(C)\). 
As a result, the median of each group is unlikely to be affected by corruption, ensuring the overall estimator remains reliable with high probability. 
The mean-of-medians estimator was originally introduced by \citet{zhong2021breaking} and later applied by \citet{xue2023efficient} to linear bandits with heavy-tailed rewards.  
Both works, however, assume symmetric underlying distributions, where the expected median equals the true mean.  
In contrast, our setting involves asymmetric Bernoulli distributions.  
To handle this asymmetry, we introduce a calibration procedure that maps the expected median to its corresponding mean, thereby producing an estimate closest to the true mean while maintaining robustness against corruption.  
Although median-based estimators are common in robust learning~\citep{lugosi2019mean}, we find that this mean-of-medians approach is particularly effective in adversarially corrupted environments, outperforming classical alternatives such as the median-of-means. 
As shown in our theoretical analysis in Section~\ref{sec:MSUCB_th}, this improvement arises from its ability to maintain low bias and strong high-probability guarantees, even when a (bounded) number of samples are corrupted.

\paragraph{\MUCB{}}

Algorithm~\ref{alg:mom_ucb_v} extends the variance-aware cascade UCB (\CUCBV{}) algorithm~\citep{vial2022minimax} by replacing the empirical mean with the robust calibrated mean-of-medians estimator and adapting it to the cascading reward setting.  
For each item \( k \in [K] \), the algorithm maintains a list of observed rewards, initialized as empty (Line~\ref{line:init_xk}).  
It first performs several rounds of pure exploration until each item has accumulated more than \(10C\) observations (Lines~\ref{line:explore1}–\ref{line:explore3}).  
After that, at each round \( t \), we estimate the mean reward using the calibrated mean-of-medians estimator (Line~\ref{line:est_mu}).
We then apply a variance-aware confidence radius \(\rho_k(t)\) around each estimate (Line~\ref{line:var_radius}).This confidence radius scales with the empirical variance of item \(k\): it contracts rapidly when the click probability is close to \(0\) or \(1\), and expands only when the feedback is genuinely noisy.  
As a result, the algorithm adaptively reduces exploration of more certain items and avoids the \(1/p\) penalty suffered by variance-unaware methods, thereby achieving lower regret.  
Finally, the upper-confidence estimate \(\bar{\mu}_k\) is formed by combining the calibrated mean and confidence radius (Line~\ref{line:ucb_k}), and the top-\(d\) items with the highest estimates are selected as the recommendation list at round \(t\) (Lines~\ref{line:select_topd}–\ref{line:play_append}).

\paragraph{Integration into Model Selection Framework.}
\MUCB{} requires knowledge of the total corruption level \(C\), which is often unavailable in practice since fraudulent click rates in \OLTR{} systems are typically unknown and time-varying.  
To eliminate this dependency, we adopt the model–selection framework of \citet{wei2022model}.  
The key idea is to instantiate a family of base learners \(\{\text{\MUCB}(C)\}_{C\in\mathcal{G}}\), each tuned for a different corruption level \(C\) drawn from a geometric grid \(\mathcal{G} = \{0, 1, 2, 4, \ldots\}\) (or any doubling schedule up to \(T\)).  
All instances operate in parallel on the same data stream: at each round, the framework selects one active instance to act, records its feedback, and updates the statistics of all remaining instances.

Periodically, the model–selection mechanism performs a statistical comparison among the active instances.  
It computes high-confidence performance estimates and eliminates any instance that is provably suboptimal relative to the current leader.  
Intuitively, instances assuming too little corruption tend to be overly optimistic—making aggressive recommendations that fail when feedback is corrupted—whereas those assuming too much corruption are overly conservative and learn too slowly.  
The elimination tests discard both extremes by identifying significant performance gaps using concentration inequalities.

Since the true corruption level \(C\) is close to some grid value \(C^* \in \mathcal{G}\), the corresponding \(\text{\MUCB}(C^*)\) (or its nearest neighbor) will survive the elimination process and eventually dominate.  
Thus, the model–selection procedure adaptively tracks the best corruption level without requiring prior knowledge of \(C\).  
The full pseudo-code and implementation details of this framework follow the construction of \citet{wei2022model}.

\section{Theoretical Analysis}\label{sec:MSUCB_th}

We now present the regret guarantees for the \MSUCB{} algorithm described in Algorithm~\ref{alg:mom_ucb_v}. 

\begin{theorem}
\label{thm:median_ucb_regret}
The expected regret of the \MUCB{} in a cascading bandits setting with corruption of level at most \(C\) is bounded by:
\begin{align}\label{eq:reg_msucb}
\Reg(T) \le O\left(KC + \sum_{k \notin S^*}\frac{\log T}{\Delta_{k}} \right).
\end{align}
\end{theorem}
\paragraph{Proof Sketch.}
We begin by analyzing the robustness of the calibrated mean-of-medians estimator under corruption. In each round, the samples associated with each item \(k\) are randomly divided into \(b\) roughly equal-sized blocks. Let \(C_{j,k}\) denote the number of corrupted samples in block \(j\). We define the event \(\mathcal{E}_k\) as the case where more than half of the samples in every block are uncorrupted, that is,  
\[
\mathcal{E}_k \coloneq \left\{\max_{1 \le j \le m} C_{j,k} < \tfrac{1}{2} \ell_{j,k}\right\}.
\]
Lemma~\ref{lem:1} shows that this event holds with high probability when the total number of samples per item satisfies \(T_{k}(t-1) > 10C\).

\begin{restatable}{lemma}{MainLemmaOne}\label{lem:1}
Assume there are \(s\) samples for a certain item \(k\).
If at most \(C\) samples are corrupted and \(s \ge 10C\), then for any constant \(\alpha>15\) we have
\begin{align}\label{eq:lem_1}
\Pr \left(\mathcal{E}_k\right) \ge 1 - s^{-3}.
\end{align}
\end{restatable}

% \begin{lemma}\label{lem:1}
% Assume there are \(s\) samples for a certain item \(k\).
% If at most \(C\) samples are corrupted and \(s \ge 10C\), then for any constant \(alpha>15\) we have
% \begin{align}\label{eq:lem_1}
% \Pr \left(\mathcal{E}_k\right) \ge 1 - s^{-3}.
% \end{align}
% \end{lemma}

The proof uses the fact that the number of corrupted samples per block follows a hypergeometric distribution, which is upper-bounded by a corresponding binomial distribution with parameter \(p = C/s \le 0.1\). 
Applying the Chernoff bound shows that the probability of having at least half of a block corrupted decays exponentially in the block size. 
Since each block contains at least \(\ell_{\min} \ge b/2 = \Omega(\log s)\) samples, a union bound over all blocks yields a total failure probability of at most \(s^{-3}\). 

Next, in Lemma~\ref{lem:2}, we show that the calibrated mean-of-medians estimator remains close to the true expected value of each item, conditioned on the event \(\mathcal{E}_k\).
\begin{restatable}{lemma}{MainLemmaTwo}\label{lem:2}
For any item \(k\), if the ''majority–uncorrupted`` event \(\mathcal{E}_k\) holds, for any \(\delta\in(0,1)\):
\begin{align}\label{eq:lem_2}
\Pr\Bigg(
  |\hat\mu_k(s)-\mu_k|
  \le \frac{1}{q_b'(\xi)}
     \Bigg[ &
       \sqrt{\frac{2b  q_b(\mu_k)\bigl(1-q_b(\mu_k)\bigr)\log(2/\delta)}{s}}
       \notag\\
       & + \frac{2b \log(2/\delta)}{3s}
     \Bigg]
\Bigg)
\ge 1-\delta, \notag \\
&\quad \text{\emph{for some} }\xi\in(0,1).
\end{align}
\end{restatable}
Here, \(q_b(p)\) denotes the expected value of the median of \(b\) samples under event \(\mathcal{E}_k\), formally defined as
\[
q_b(p) \coloneq \Pr\big(\mathrm{Bin}(b,p)\ge (b{+}1)/2\big).
\]
To prove this, we use Bernstein's inequality to show that the empirical mean of these block medians concentrates around \(q_b(\mu_k)\) with high probability, yielding a deviation bound in terms of \(s\) and \(b\). 
Applying the mean value theorem to the inverse calibration function \(g_b = q_b^{-1}\) then transfers this concentration result from the transformed \(q\)-space back to the original mean space, introducing a scaling factor of \(1/q_b'(\xi)\). 
Combining these steps gives the final high-probability bound on \(|\hat{\mu}_k(s) - \mu_k|\) stated in~\ref{eq:lem_2}.
Then, in Lemma~\ref{lem:3}, we rewrite the estimation error bound, replacing the unknown variance with empirical data variance 
\begin{restatable}{lemma}{MainLemmaThree}\label{lem:3}
Fix a round \(t\) and an item \(k\) and let \(s=T_k(t-1)\) for item \(k\). Set \(a_t \coloneq \frac{\log t}{s}\) and \(x \coloneq |\hat\mu_k(s)-\mu_k|\). 
Under the majority–uncorrupted event \(\mathcal{E}_k\), there exist constants \(A_s\) and \(B_s\) such that
such that with probability at least \(1-t^{-4}\),
\begin{align}\label{eq:lem_3}
x \le A_s \sqrt{\hat v_k(s) a_t} + B_s a_t,
\qquad
\hat v_k(s)\coloneq \hat\mu_k(s)\bigl(1-\hat\mu_k(s)\bigr).
\end{align}
\end{restatable}
To prove Lemma~\ref{lem:3}, we use the fact that the Bernoulli variance map \(v(p)=p(1-p)\) is \(1\)-Lipschitz on \([0,1]\), i.e., \(|u(1-u)-v(1-v)|\le |u-v|\). Substituting \(u=\mu_k\), \(v=\hat\mu_k(s)\), and noting \(x=|\hat\mu_k(s)-\mu_k|\), we obtain \(\mu_k(1-\mu_k)\le \hat v_k(s)+x\), which replaces the unknown variance by an empirical term in the variance-aware bound.

Combining \eqref{eq:lem_3} with the UCB radius in Line~\ref{line:var_radius}, we obtain a high–probability \emph{sufficient} condition under which the UCB index of a suboptimal item \(k\) falls below that of an optimal item \(k^*\):
\[
A\sqrt{\tfrac{\hat v_k(s) \log t}{s}}  +  B\tfrac{\log t}{s}
 \le  \tfrac{1}{2} \Delta_{k,k^*}.
\]
This yields the sample–size threshold
\[
s  \ge  \frac{16 A^{2} \hat v_k(s) \log t}{\Delta_{k,k^*}^{2}}
 +  \frac{4 B \log t}{\Delta_{k,k^*}}.
\]
Using this threshold, we bound the number of rounds in which a suboptimal item \(k\) can be mistaken for \(k^*\); applying Lemma~1 of \citet{kveton2015cascading} then gives the stochastic regret term \(O \left(\sum_{k\notin S^*}\frac{\log T}{\Delta_k}\right)\). 
Finally, since each round reveals at least one item’s feedback, within at most \(10KC\) rounds every item accumulates \(10C\) observations. 
Combining these parts yields the final bound stated in Theorem~\ref{thm:median_ucb_regret}. \qed

The full proof of Theorem~\ref{thm:median_ucb_regret} is provided in the Appendix~\ref{sec:appendix_A} and~\ref{sec:appendix_B}.

\begin{theorem}[Regret of \MSUCB{}]
\label{thm:msucb_regret}
Let \(\Delta\) denote the minimum gap between any suboptimal item and any optimal item. Then \MSUCB{} satisfies
\[
\mathrm{Reg}(T)=O \left(KC+\frac{K\log T}{\Delta}\right).
\]
\end{theorem}

\begin{proof}
By Theorem~4 in \cite{wei2022model}, if \texttt{G\textendash COBE} is run with a base learner whose regret is
\(O \left(\frac{\beta_1}{\Delta}+\beta_2 C\right)\),
then the model selection procedure achieves regret
\(O \left(\frac{\beta_4}{\Delta}+\beta_2 C\right)\),
where \(\beta_4=10^{4} \left(2\beta_1+42\beta_2\log T\right)\).
By Theorem~\ref{thm:median_ucb_regret}, our base learner yields \(\beta_1=K\log T\) and \(\beta_2=K\).
Substituting these values and absorbing constants gives
\[
\mathrm{Reg}(T)=O \left(KC+\frac{K\log T}{\Delta}\right).
\]
which proves the claim.
\end{proof}

\paragraph{Discussion} We compare out algorithm \MSUCB{} with the lower bound and other baselines in the following two settings.

(i) \textbf{Uncorrupted setting.}  
The gap–dependent lower bound for cascading bandits problem was established by \citet{kveton2015cascading} as \(O \left( O\left(\sum_{k \notin S^*}\frac{\log T}{\Delta_{k}} \right)\right)\), and the gap–independent lower bound was given by \citet{vial2022minimax} as \(O \left( \sqrt{KT}\right)\). 
In the absence of corruption, we prove that our algorithm attains the gap–dependent lower bound. The gap–independent guarantee follows by adapting the proof of Theorem~2 in \cite{vial2022minimax} and substituting our mean estimator for the original one. In contrast, both \texttt{CascadeRAC}~\cite{xie2025cascading} and \texttt{FORC}~\cite{Golrezaei2021Learning} are suboptimal by a factor of \(d\).  
We further extend \CBARBAR{}~\cite{xu2021simple} from linear combinatorial rewards to the cascading setting, showing that the resulting \CBARBARNL{} algorithm is suboptimal by a factor of \(d^2\). 
Moreover, all aforementioned algorithms incur an additional \(\log T\) factor.  
The pseudocode, regret bounds, and proof sketch for \CBARBARNL{} are provided in Appendix~\ref{sec:appendix_C}.  
Overall, \MSUCB{} outperforms all existing robust cascading bandit algorithms in the stochastic regime.

(ii) \textbf{Corrupted setting.}  
To the best of our knowledge, no prior work establishes a lower bound for cascading bandits under corruption. In the linear combinatorial bandit setting, the corruption lower bound is \(O(C)\). 
Since cascade feedback provides strictly less information than semi–bandit feedback, the corresponding lower bound for cascading bandits must be at least \(O(C)\).  
Under corruption level \(C\), both \texttt{CascadeRAC} and \texttt{FORC} suffer multiplicative corruption terms in their regret. 
Our extension, \CBARBARNL{}, achieves an additive \(O(dC)\) corruption term, and our proposed algorithm similarly maintains a linear additive corruption term. Although this remains suboptimal by a factor of \(K\) compared to the ideal \(O(C)\) rate, our experiments show that \MSUCB{} consistently outperforms \CBARBARNL{} by a substantial margin in practice.

\section{Empirical Evaluation}
\label{sec:NumRes}
\begin{figure*}[t]
  \centering
\begin{subfigure}[t]{0.3\linewidth}
     \centering    \includegraphics[width=\linewidth]{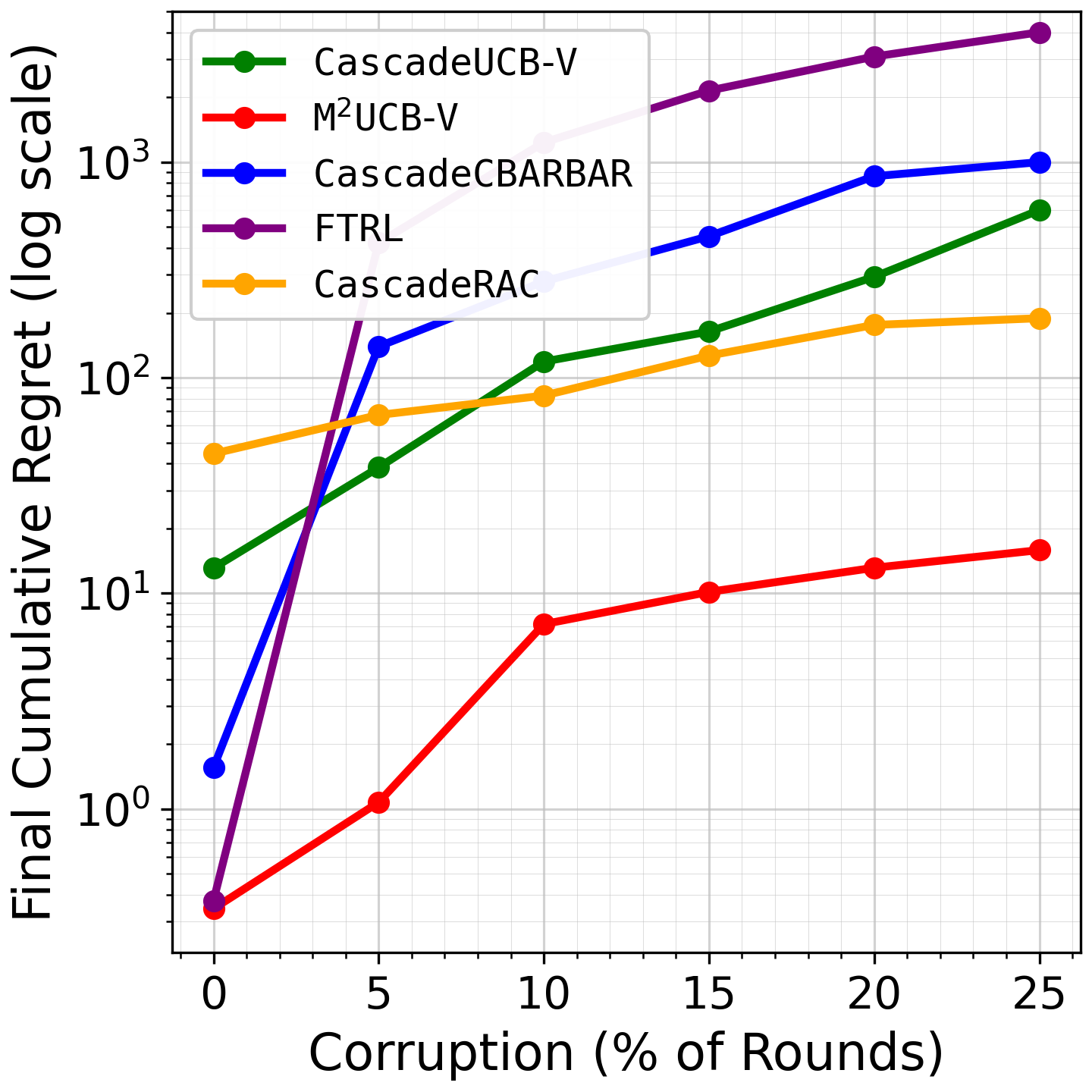}
   \caption{\small Yelp Dataset}
     \label{fig:d2}    \end{subfigure}\hfill
 \begin{subfigure}[t]{0.3\linewidth}
     \centering  \includegraphics[width=\linewidth]{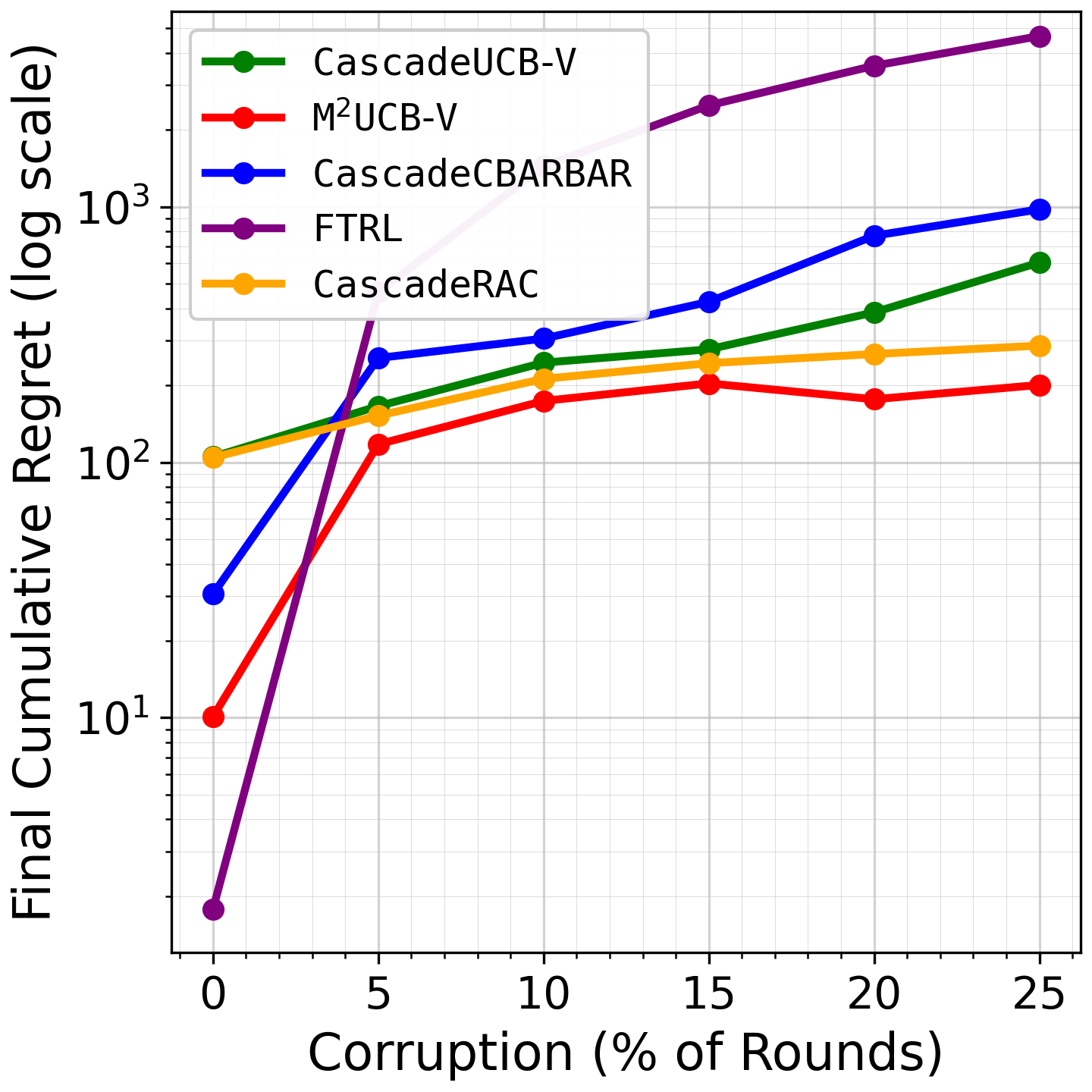}    \caption{\small MovieLens Dataset}
     \label{MovieLens} \end{subfigure}\hfill
\begin{subfigure}[t]{0.3\linewidth}
     \centering  \includegraphics[width=\linewidth]{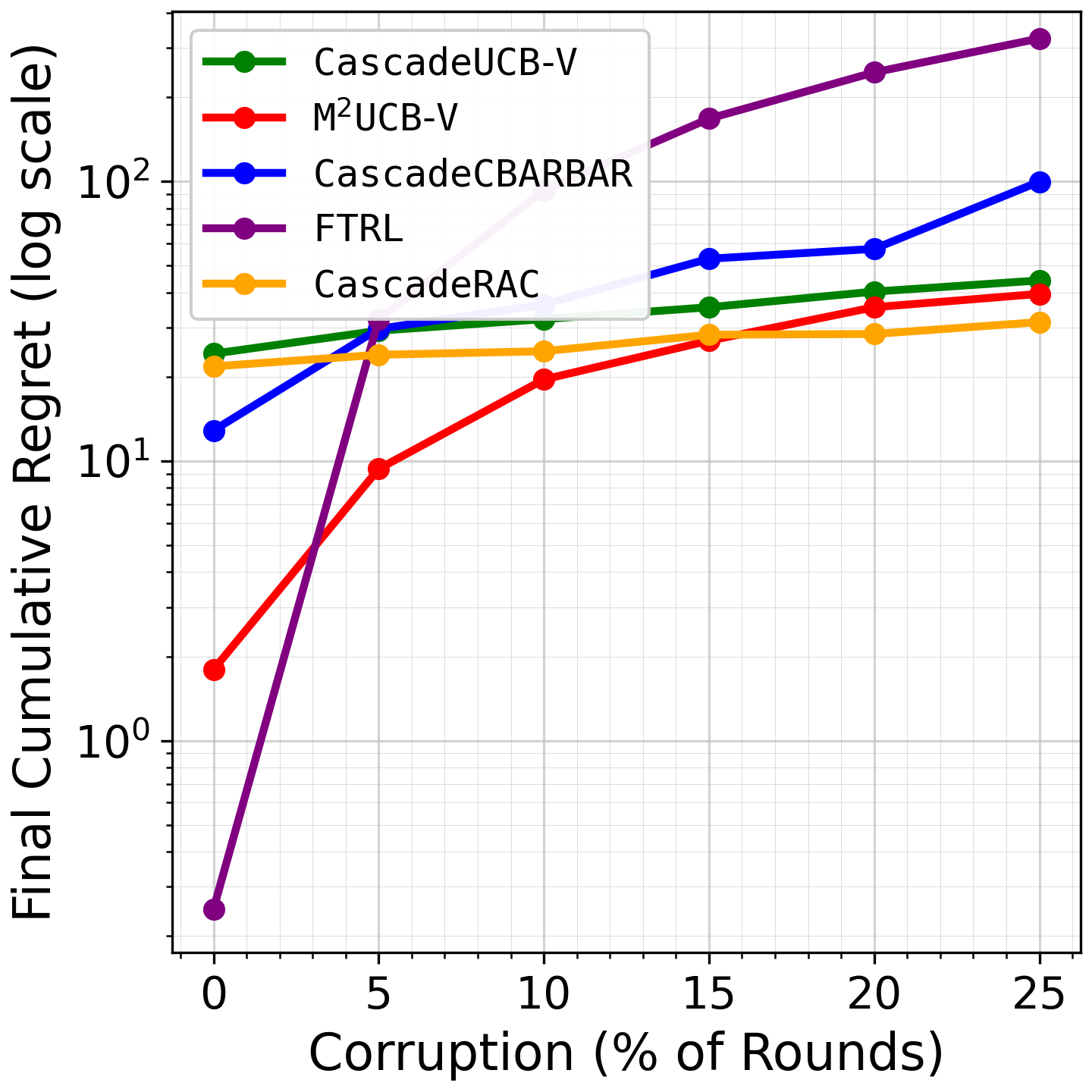}
  \caption{\small LastFM Dataset}%
     \label{fig:LastFM}
   \end{subfigure}
   \vspace{-2mm}
\caption{Comparing final cumulative regret of the algorithms after 40K rounds with list size $d=10$.}   \label{fig:CumulativeRegret}
\end{figure*}

\begin{figure*}[t]
  \centering
\begin{subfigure}[t]{0.245\linewidth}\label{fig:2_5corr}
     \centering    \includegraphics[width=\linewidth]{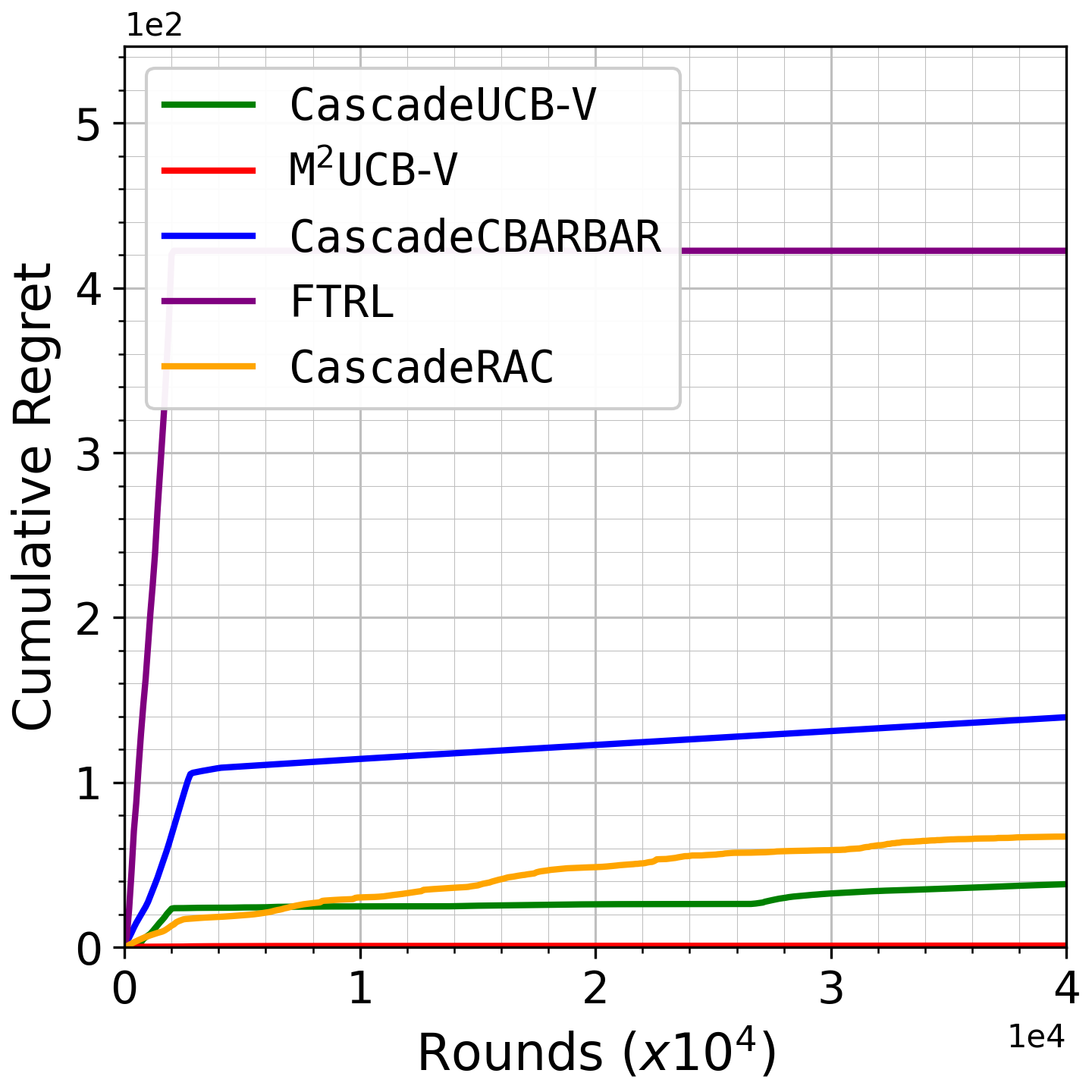}
   \caption{\small Corruption = 5\%}
 \end{subfigure}\hfill
 \begin{subfigure}[t]{0.245\linewidth}\label{fig:2_10corr}
     \centering  \includegraphics[width=\linewidth]{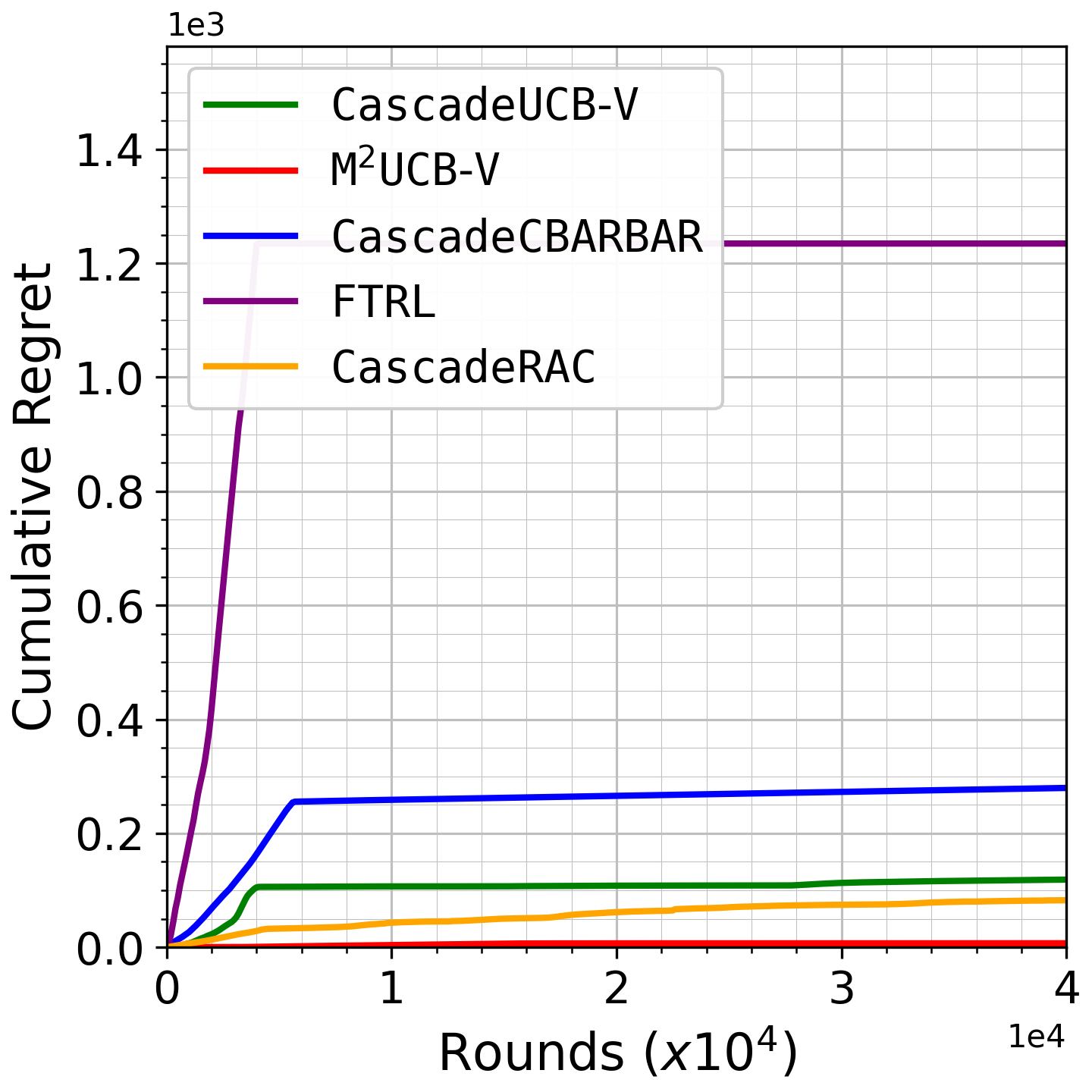}    \caption{\small Corruption=10\%}
  \end{subfigure}\hfill
\begin{subfigure}[t]{0.245\linewidth}\label{fig:2_15corr}
     \centering  \includegraphics[width=\linewidth]{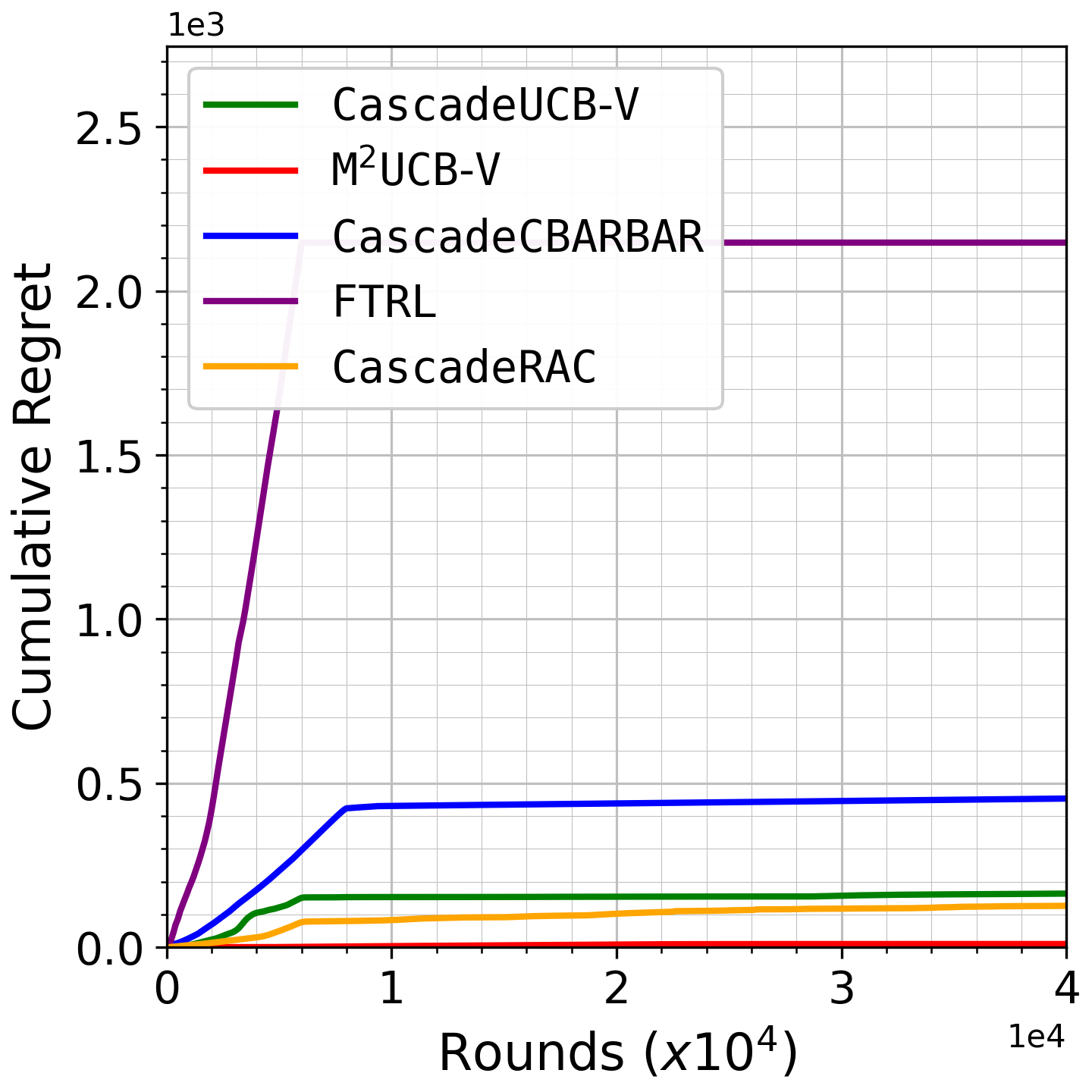}
  \caption{\small Corruption=15\%}
   \end{subfigure}
   \begin{subfigure}[t]{0.245\linewidth}\label{fig:2_20corr}
     \centering  \includegraphics[width=\linewidth]{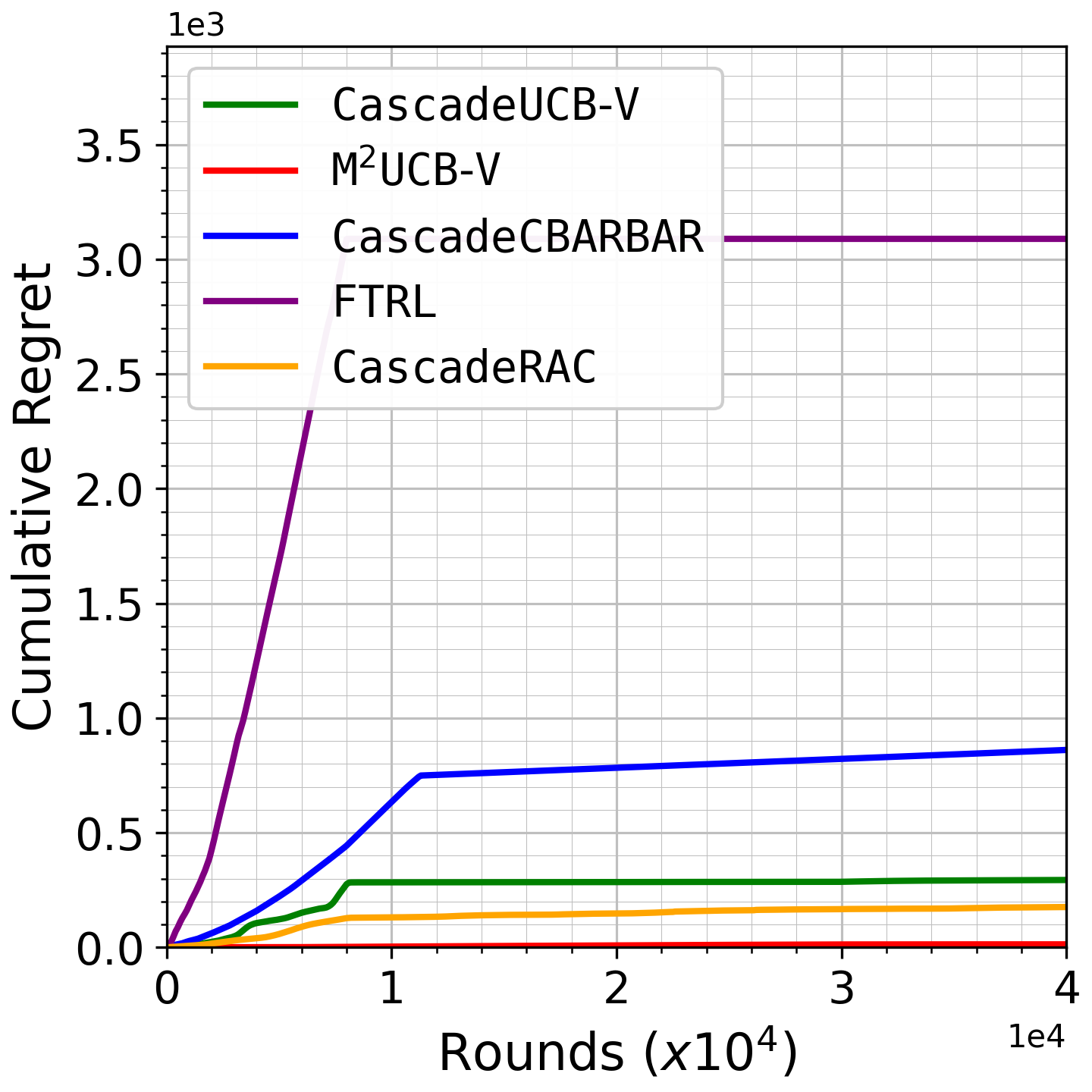}
  \caption{\small Corruption= 20\%}
   \end{subfigure}
   \vspace{-2mm}
\caption{Growth of cumulative regret as a function of rounds for the Yelp dataset, with list size $d=10$.}   \label{fig:CumulativeRegretRoundsY}
\end{figure*}

\begin{figure*}[t]
  \centering
\begin{subfigure}[t]{0.245\linewidth}\label{fig:2_5corr}
     \centering    \includegraphics[width=\linewidth]{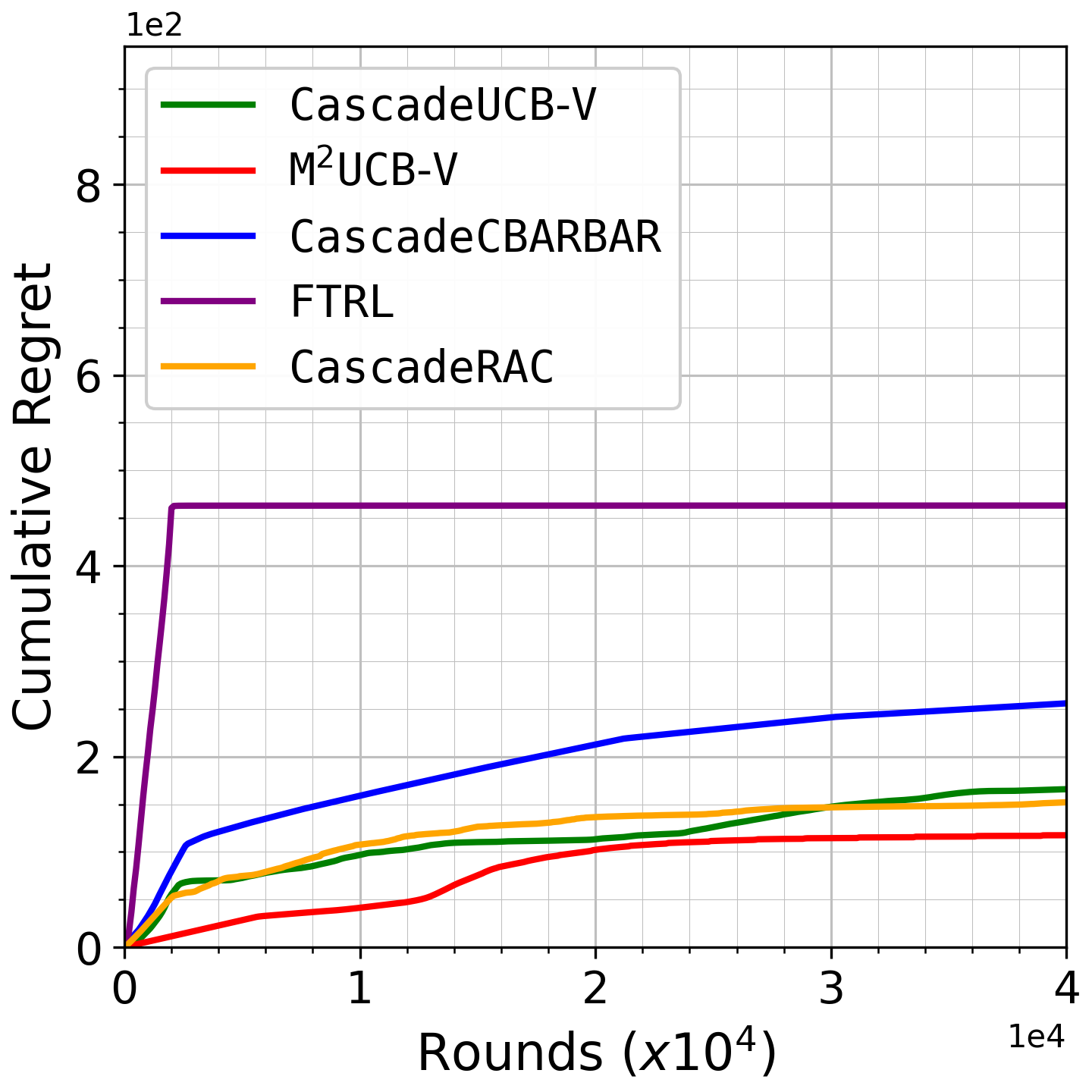}
   \caption{\small Corruption = 5\%}
 \end{subfigure}\hfill
 \begin{subfigure}[t]{0.245\linewidth}\label{fig:2_10corr}
     \centering  \includegraphics[width=\linewidth]{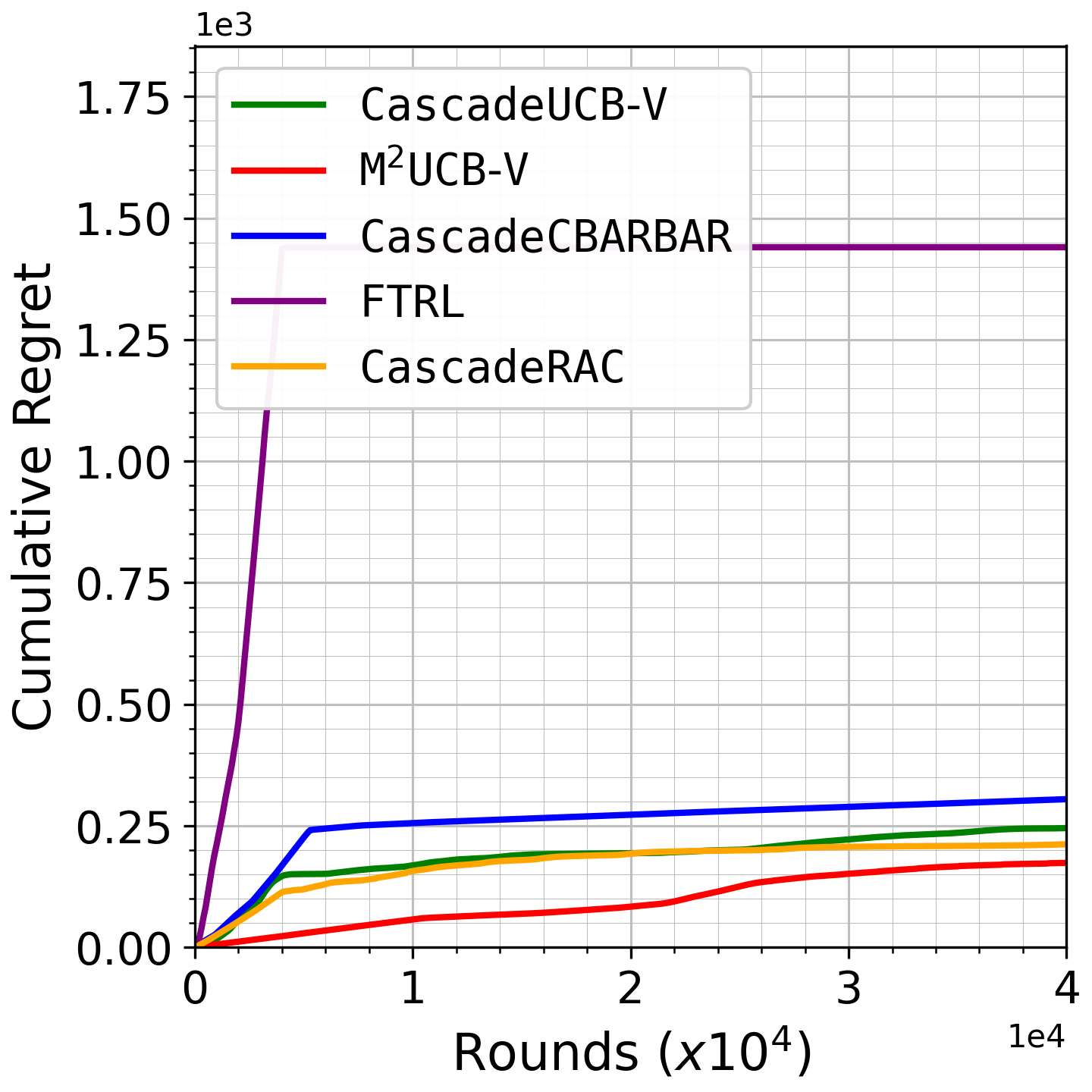}    \caption{\small Corruption=10\%}
  \end{subfigure}\hfill
\begin{subfigure}[t]{0.245\linewidth}\label{fig:2_15corr}
     \centering  \includegraphics[width=\linewidth]{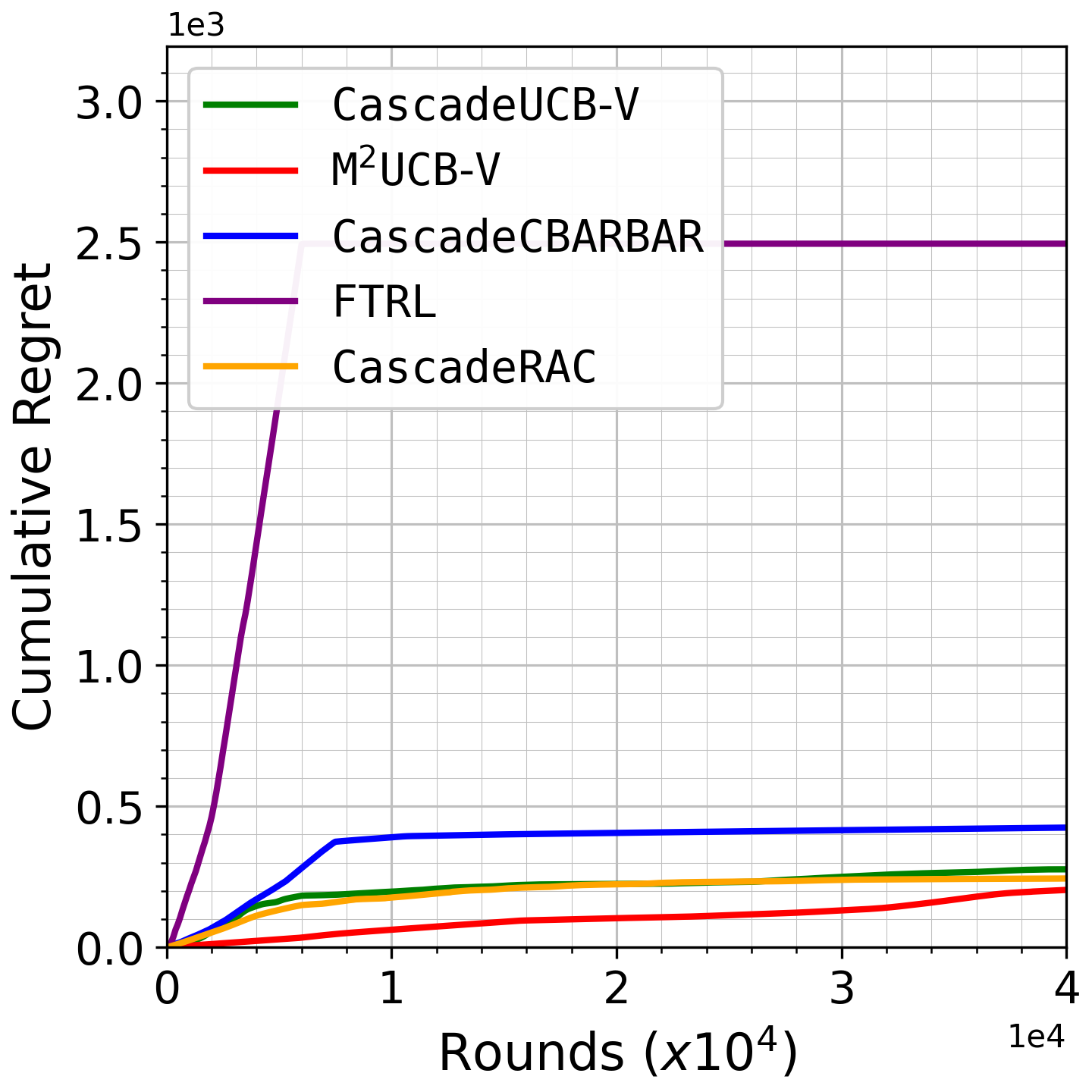}
  \caption{\small Corruption=15\%}
   \end{subfigure}
   \begin{subfigure}[t]{0.245\linewidth}\label{fig:2_20corr}
     \centering  \includegraphics[width=\linewidth]{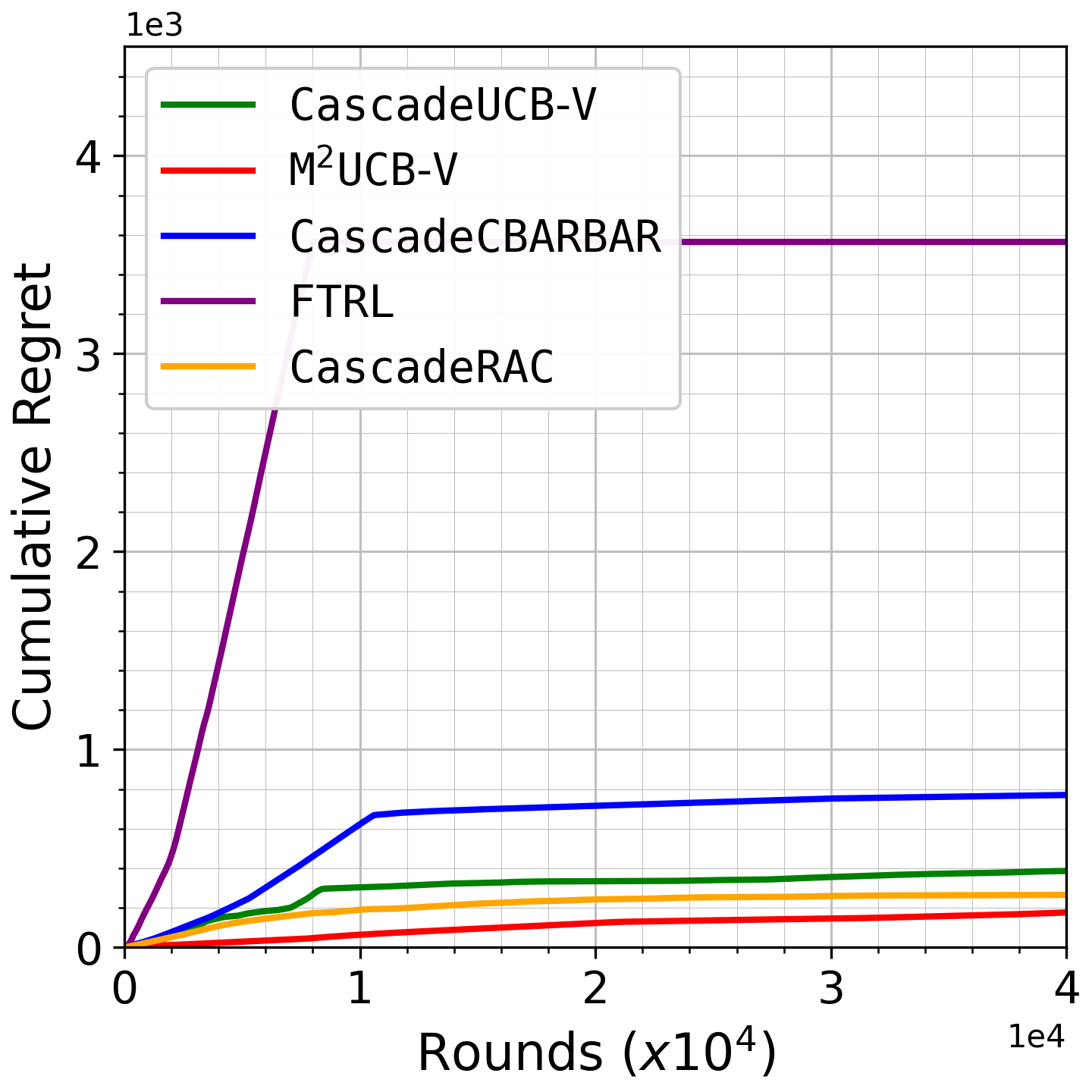}
  \caption{\small Corruption= 20\%}
   \end{subfigure}
   \vspace{-2mm}
\caption{Growth of cumulative regret as a function of rounds for the MovieLens dataset, with list size $d=10$.}   \label{fig:CumulativeRegretRoundsML}
\end{figure*}

\begin{figure*}[t]
  \centering
\begin{subfigure}[t]{0.245\linewidth}\label{fig:2_5corr}
     \centering    \includegraphics[width=\linewidth]{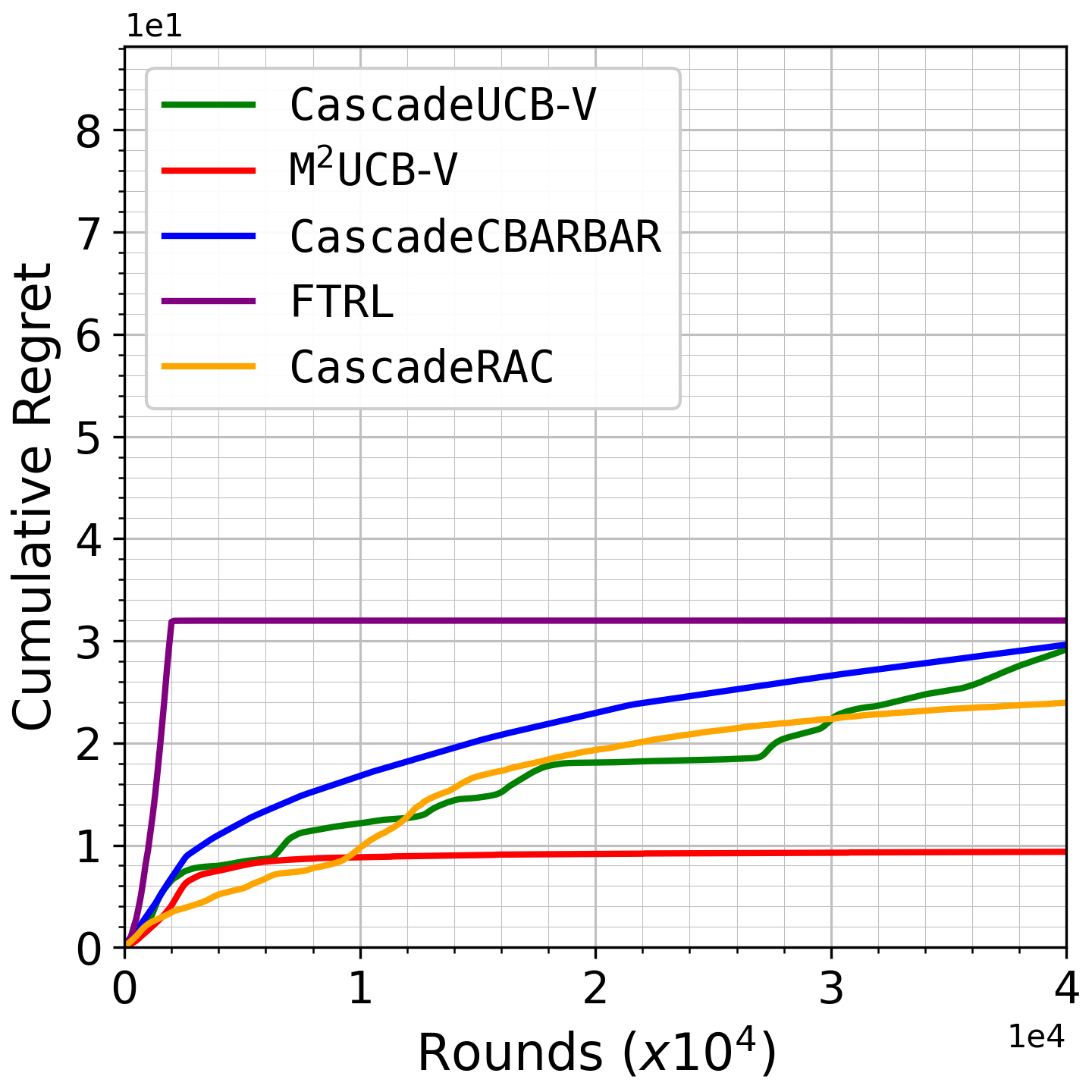}
   \caption{\small Corruption = 5\%}
 \end{subfigure}\hfill
 \begin{subfigure}[t]{0.245\linewidth}\label{fig:2_10corr}
     \centering  \includegraphics[width=\linewidth]{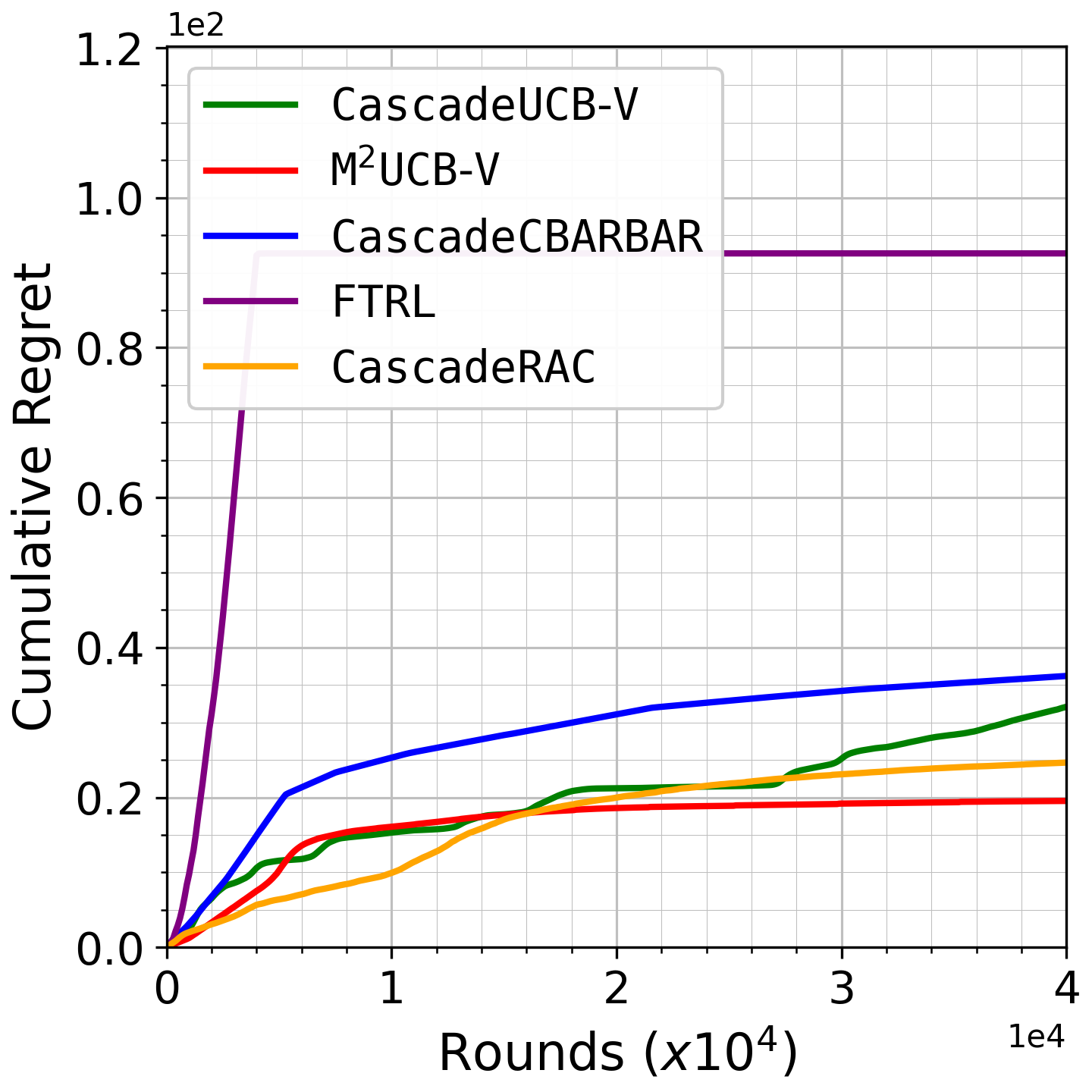}    \caption{\small Corruption=10\%}
  \end{subfigure}\hfill
\begin{subfigure}[t]{0.245\linewidth}\label{fig:2_15corr}
     \centering  \includegraphics[width=\linewidth]{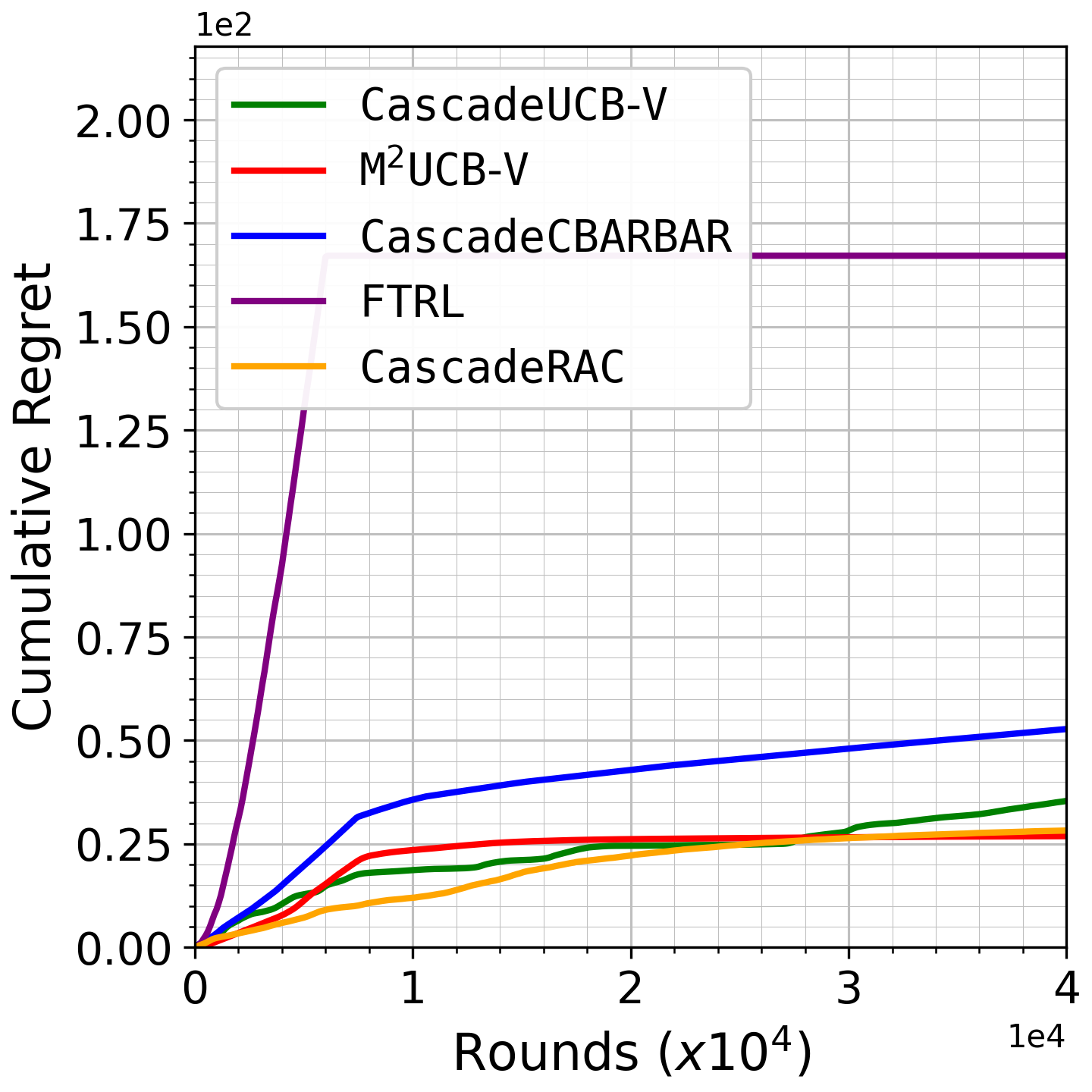}
  \caption{\small Corruption=15\%}
   \end{subfigure}
   \begin{subfigure}[t]{0.245\linewidth}\label{fig:2_20corr}
     \centering  \includegraphics[width=\linewidth]{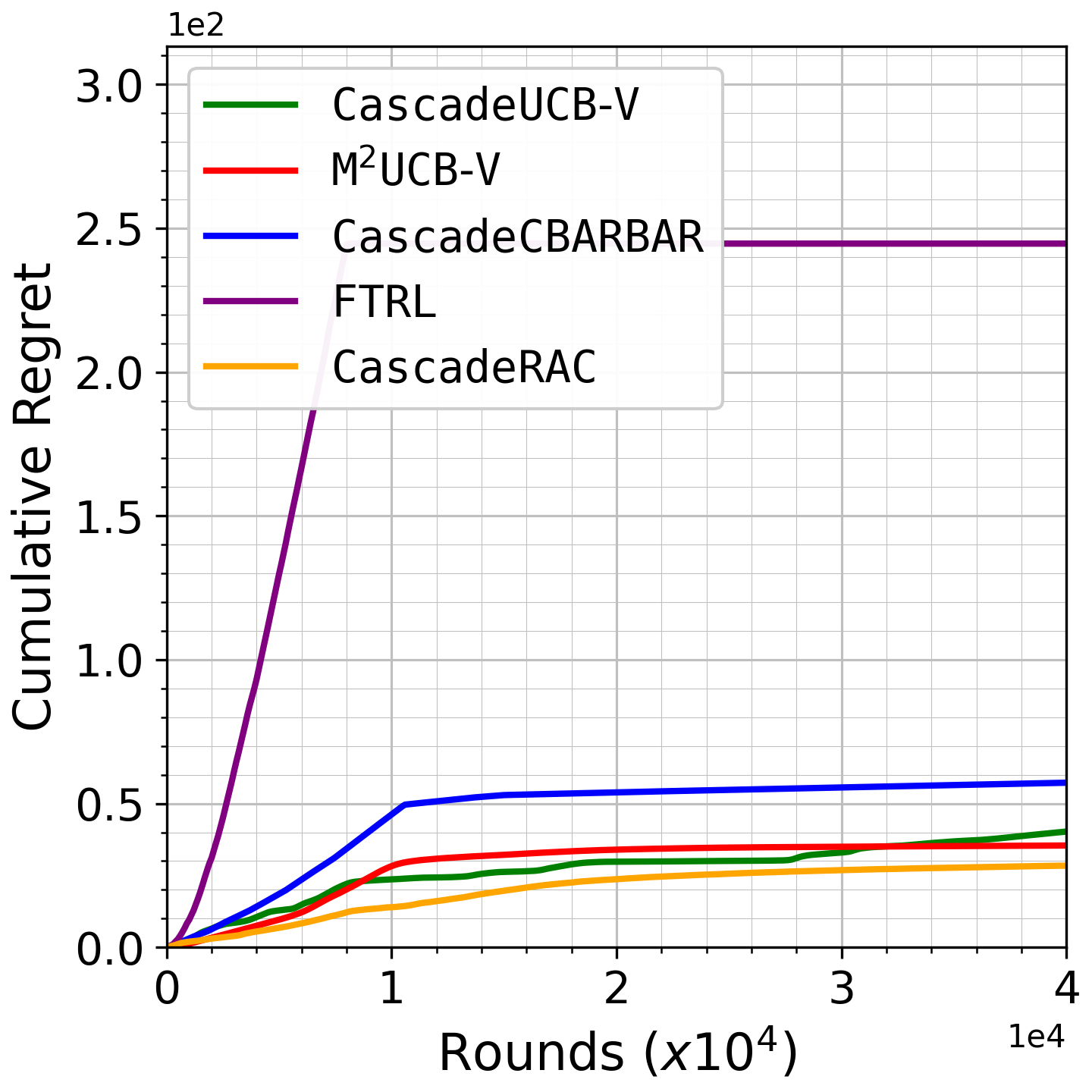}
  \caption{\small Corruption= 20\%}
   \end{subfigure}
   \vspace{-2mm}
\caption{Growth of cumulative regret as a function of rounds for the LastFM dataset, with list size $d=10$.}   \label{fig:CumulativeRegretRoundsLFM}
\end{figure*}

We validate our approach through extensive experiments on the following three large-scale real-world datasets, which together represent some of the most common applications of \OLTR{}. 

\paragraph{Datasets} The Yelp dataset~\cite{yelp_open_dataset} consists of 6,990,280 ratings  of 150,346 businesses from 1,987,897 users.
The MovieLens dataset~\cite{Harper2015MovieLens} is derived from a movie recommendations website and consists of 32,000,204 ratings of 87,585 movies from 200,948 users. 
Finally, the LastFM dataset~\cite{Schedl2016LFM} is derived from an online music site where users can tag artists and consists of 359,347 users and 186,642 artists. 

\paragraph{Baseline algorithms} We compare \MSUCB{} against strong baselines: \CUCBV{}~\cite{vial2022minimax}, which is optimal in the clean stochastic setting but not robust; \FTRL{}~\cite{Ito2021Advances}, a best-of-both-worlds combinatorial bandit method; \CRAC{}~\cite{xie2025cascading}, a corruption-robust cascading algorithm whose uncorrupted regret is suboptimal and whose dependence on corruption is multiplicative; and \CBARBARNL{}, our extension of \CBARBAR{}~\cite{xu2021simple}, which achieves an additive corruption term but remains suboptimal in the uncorrupted regime. We omit \texttt{FORC}~\cite{Golrezaei2021Learning} from experiments because its methodology and regret guarantees are similar to \CRAC{} and slightly weaker. A head-to-head theoretical comparison of all baselines appears in Table~\ref{tab:algo_comparison}. 

\paragraph{Bandit Arms} To test our algorithms, we define the bandit arms for each of the three application domains as follows. For Yelp, we ran the query ``pizza in Philadelphia'' which produced 529 restaurants, each of which was a bandit arm in our simulation. For MovieLens and LastFM, we chose a random set of 500 movies and artists, respectively, each with at least 100 ratings from users. 

\paragraph{Computing click probabilities of each arm from user ratings} 
For each arm corresponding to an item in the dataset, we compute the average rating and the number of ratings of the item. 
We then compute the click probability of the arm using standard techniques used in recommender systems like IMDB \cite{imdb_weighted_rating} as follows. First, we use the Bayesian average technique \cite{masurel_bayesian_average_star_ratings, yang2013combining} to compute the rating of each arm to account for the fact that the true rating is closer to the average value when more ratings are averaged. We then convert the Bayesian-averaged rating of each arm to a click probability using a sigmoid function, a technique that is commonly used in recommender systems \cite{cheng_wide_deep_2016}.

\paragraph{Corruption model} Given a corruption budget, the adversary selects \(C\) rounds to corrupt and flips the observed click outcomes until the budget is exhausted. 
As per the definition in~\eqref{eq:c}, any number of item-wise reward flips within a single round consumes only one unit of corruption. 
Consequently, to maximize impact under a fixed budget, we invert all observed rewards in corrupted rounds. 
We place these corrupted rounds at the very beginning of the horizon to disrupt exploration and induce prolonged exploitation of suboptimal items. 
We vary the corruption rate from \(0\%\) to \(25\%\), which is higher than typical rates reported in practice. 
For context, an analysis of Yelp found that approximately \(16\%\) of reviews are flagged as fake by the platform's filter \cite{luca2016fake}, and Amazon's 2023 transparency report indicates roughly \(250\) million fake ratings out of an estimated \(2.5\) billion total, about \(10\%\) \cite{amazon_transparency_report_2023}.

% \paragraph{Choosing the group size of \MSUCB{}} A key design choice for our \MSUCB\ algorithm is the group size that is used for finding the median. We ran \MSUCB\ with three different methods for picking the group size: $\log(T)$, $\left\lceil \log(T)/2\right \rceil$, and $\left\lceil 2 \log(T)\right \rceil$, where $T$ is the number of rounds. The corruption percentage is the percentage of rounds where the adversary corrupted all the base arms in that round. 
% Figure~\ref{fig:groupsize} shows that the largest group size of $ \left\lceil 2 \log(T)\right \rceil$ performs the best for all non-zero corruption percentages and performs reasonably well when there is no corruption. Therefore, we use the group size of $\left\lceil 2 \log(T)\right \rceil$ for further experiments with \MSUCB{}.

\paragraph{Final cumulative regret for different corruption percentages.}
Figure~\ref{fig:CumulativeRegret} reports the final cumulative regret after \(40\mathrm{K}\) rounds across corruption levels. 
In the no-corruption case, \MSUCB{} is among the top performers, consistent with its optimal stochastic regret. 
As corruption increases, \MSUCB{} degrades most gracefully: it outperforms all baselines on Yelp and MovieLens for \(5\%\)–\(25\%\) corruption, and on LastFM for \(5\%\)–\(15\%\). 
At very high corruption on LastFM (\(\ge 20\%\)), \CRAC{} edges out \MSUCB{} by a small margin.
By contrast, \CRAC{} performs poorly in purely stochastic settings, where it is consistently the worst baseline, suggesting a systematic overestimation of corruption. 
\CUCBV{} unexpectedly outperforms \CBARBARNL{} in several settings, likely because \CBARBARNL{}'s epoch-based schedule delays updates and prolongs exploration, hurting early performance.
\FTRL{} excels with no corruption, quickly finding the optimal list, but its regret rises steeply even at \(5\%\) corruption as early corrupted feedback steers it toward low-reward lists; it is the worst performer beyond \(5\%\) on all three datasets.

\paragraph{Cumulative regret as a function of rounds.} Figure~\ref{fig:CumulativeRegretRoundsY} shows cumulative regret over \(40k\) rounds on Yelp at several corruption levels. 
\MSUCB{} converges fastest in both low corruption setting and under heavy corruption, consistent with our theory.
\CBARBARNL{} exhibits a sharp convergence point that occurs much later than \CUCBV{}'s. This happens at the end of an epoch. In our runs, corruption ceases mid-epoch, but \CBARBARNL{} updates only at epoch boundaries. As a result, fresh observations are not incorporated immediately, exploration persists longer than necessary, and cumulative regret increases even though the method is robust to corruption.

Figures~\ref{fig:CumulativeRegretRoundsML} and~\ref{fig:CumulativeRegretRoundsLFM} show analogous trajectories over \(40K\) rounds on Yelp and LastFM at the same corruption levels. 
The algorithms exhibit similar qualitative behavior across datasets, corroborating the robustness and consistency of \MSUCB{}. 
As discussed earlier, \CRAC{} can catch up at the highest corruption rates. 
However, at low corruption levels, \MSUCB{} remains competitive with stochastic-optimal baselines, whereas \CRAC{} is suboptimal in the low-corruption regimes by a wide margin.

\section{Conclusion}

We study click fraud in \OLTR{} through the lens of cascading bandits with corruption. 
We introduce \MSUCB{}, a corruption-robust algorithm that integrates three key components: (1) a calibrated mean-of-medians estimator to obtain robust estimates and filter out corrupted feedback, (2) a variance-aware UCB radius to achieve optimal stochastic regret, and (3) a model selection wrapper that removes the need to know the corruption level in advance. 
Theoretically, \MSUCB{} has an optimal regret in the stochastic regime and an additive \(O(KC)\) corruption term. 
Empirically, \MSUCB{} consistently outperforms strong baselines on Yelp, MovieLens, and LastFM, even under substantial corruption.
As a secondary contribution, we extend \CBARBAR{} to the cascading setting, which is termed \CBARBARNL{}. We analyze its regret and use it as a competitive robust baseline.

This work closes an important gap by providing, to our knowledge, the first corruption-robust algorithm for cascading bandits that simultaneously achieves the stochastic optimum, is agnostic to the corruption level, and performs strongly in practice while incurring only an additive dependence on \(C\). 
Promising directions include tightening the additive term from \(O(KC)\) toward the \(O(C)\) benchmark, and developing matching lower bounds for cascading bandits under corruption.

% \clearpage
\bibliographystyle{plainnat}
\bibliography{sample-sigconf}
%%%%%%%%%%%%%%%%%%%%%%%%%%%%%%%%%%%%%%%%%%%%%%%%%%%%%%%%%%%%

\clearpage
\appendix
\section{Proof of Lemmas~\ref{lem:1} and~\ref{lem:2}}
\label{sec:appendix_A}
In this appendix we provide the proof to Lemmas~\ref{lem:1} and~\ref{lem:2}.
\paragraph{Setting and notation.}
There are \(K\) items (items) with unknown means \(\mu_k\in(0,1)\).
At each round \(t=1,2,\dots,T\), the learner recommends an ordered list
\(S_t=(k^{(1)},\dots,k^{(d)})\). The user scans the list top-down and clicks the
first attractive item. The round reward is \(R_t=\mathbf 1\{\text{at least one click}\}\)
with expectation \(r(S_t)=1-\prod_{k \in S_t} (1-\mu_{k})\).
Let \(S^\star\) be an optimal list and write the \(d\)-th largest mean as
\(\mu^{(d)}\). For a suboptimal item \(k \notin S^\star\), define the gap as
\(
\Delta_k  \coloneq  \mu^{(d)}-\mu_k > 0.
\)
Let \(T_k(t)\) be the number of observations of \(k\) up to time \(t\),
i.e., the number of rounds in which the agent pulled \(k\).
We adopt a corruption model where an oblivious adversary may flip observed feedback bits, let \(c_k(s)\) be the number of flips among the first \(s\) observations of \(k\) The total budget satisfies \(c_k(T) \le C\) for every item \(k\).

\paragraph{Estimator (mean of medians with calibration).}
After \(s=T_{k}(t-1)\) observations of \(k\), partition the \(s\) bits into
\(m=\lfloor s/b \rfloor\) consecutive blocks of size \(b\) uniformly at random, with \(b = \lceil \alpha \log s \rceil\), and \(b\) is adjusted to be odd by \(\pm1\) if needed (this affects only constants)..
Let \(M_{j,k}(s)\in\{0,1\}\) be the median in block \(j\),
and define the mean of medians \(\overline M_k(s)=\tfrac{1}{m}\sum_{j=1}^m M_{j,k}(s)\).
For \(p\in[0,1]\), set
\begin{align}\label{eq:q}
q_b(p) \coloneq \Pr \big(\mathrm{Bin}(b,p)\ge (b{+}1)/2\big),\qquad g_b \coloneq q_b^{-1}.
\end{align}
The calibrated estimator is
\begin{align}\label{eq:calib_est}
\hat \mu_k(s) \coloneq g_b\big(\overline M_k(s)\big),
\qquad
\hat v_k(s) \coloneq \hat \mu_k(s)\big(1-\hat \mu_k(s)\big).
\end{align}

Let \(\ell_{j, k}\in\{\lfloor s/m\rfloor,\lceil s/m\rceil\}\) be the size of block \(j\in[m]\) and
let \(C_{j, k}\) be the number of corrupted samples falling into block \(j\).
We define \(\mathcal{E}_k\) as an event where more than half of the samples in all blocks are uncorrupted, \(\mathcal{E}_k \coloneq \{\max_{1\le j\le m} C_{j, k}  <  \tfrac{1}{2} \ell_{j, k}\}\).

\MainLemmaOne*
 
\begin{proof}
Conditioned on the oblivious set of the \(C\) corrupted indices.
For each block \(j\), the count \(C_{j, k}\) has a hypergeometric distribution with parameters \((N,K,n)=(s,C,\ell_{j, k})\),
and the hypergeometric upper tail is dominated by the corresponding binomial tail, since replacement only increases the variance:
\begin{align*}
\Pr(C_{j, k} \ge t) \le \Pr \left(\mathrm{Bin}(\ell_{j, k},p)\ge t \right),\qquad p\coloneq C/s \le \nicefrac{1}{10}.
\end{align*}

\paragraph{Single–block tail at the half–block threshold.}
Set an item \(k\) and \(\ell'=\lceil \ell_{j, k}/2\rceil\) and write \(m_j=\mathbb E[\mathrm{Bin}(\ell_{j, k},p)]=\ell_{j, k} p\).
Define \(\delta\) gap
\[
\delta  \coloneq  \frac{\ell'-m_j}{m_j} \ge \frac{\frac{1}{2}\ell_{j, k}-\ell_{j, k} p}{\ell_{j, k} p} = \frac{\tfrac{1}{2}-p}{p}.
\]
By the Chernoff bound for binomials,
\begin{align}\label{eq:single_bound}
\Pr  \left(\mathrm{Bin}(\ell_{j, k},p)\ge (1+\delta)m_j \right)
 \le \exp  \left(- \frac{\delta^2}{3} m_j \right)
\overset{(a)}{\le} \exp  \left(- \tfrac{8}{15} \ell_{j, k} \right),
\end{align}
Where in (a) we replace \(p\) with its upper bound \(0.1\).

\paragraph{A uniform lower bound on block sizes.}
With $m=\lceil s/b\rceil$ and $b=\lceil \alpha\log s\rceil$, for $s$ large enough we have $s\ge 2b$,
and hence
\begin{align*}
\ell_{j, k} \ge \Big\lfloor \frac{s}{m}\Big\rfloor \ge \Big\lfloor \frac{s}{\lceil s/b\rceil}\Big\rfloor
\ge \Big\lfloor \frac{s}{s/b+1}\Big\rfloor \ge \Big\lfloor \frac{b s}{s+b}\Big\rfloor \ge \frac{b}{2}
\eqcolon  \ell_{\min}.
\end{align*}

\paragraph{Union bound over blocks.}
Using \eqref{eq:single_bound} and the bound on \(\ell_{\min}\),
\begin{align}
\Pr  \left(\exists j\in[m]: C_{j, k} \ge \tfrac{1}{2}\ell_{j, k} \right)
\le& m\cdot \exp \left(- \frac{8}{15}\ell_{\min}\right)\notag\\
\le& \left(\frac{s}{b}+1\right)\exp \left(- \frac{4}{15}b\right)\notag\\
\le& \left(\frac{2s}{b}\right)\exp \left(- \frac{4}{15}\alpha\log s\right)\notag\\
\le& \frac{2}{\alpha\log s}  s^{ 1-\frac{4}{15}\alpha} \overset{(a)}{\le} s^{-3}.
\end{align}
Where (a) is due to the fact that we choose \(\alpha>15\), therefore \(1-\frac{4}{15}\alpha\le -3\).
Therefore, with probability at least $1-s^{-3}$, every block has strictly fewer than $\ell_{j, k}/2$
corrupted samples. 
\end{proof}

\MainLemmaTwo*

\begin{proof}
\emph{Step 1 (distribution under the good event).}
Assuming event \(\mathcal{E}_k\), each block median has Bernoulli distribution, hence due to \eqref{eq:q}
\begin{align*}
M_{j, k} \sim \mathrm{Ber} \left(q_b(\mu_k) \right),
\qquad
\{M_{j, k}\}_{j=1}^m \text{ are independent}.
\end{align*}

emph{Step 2 (concentration in \(q\)-space).}
Bernstein’s inequality for bounded variables gives, for any \(\delta\in(0,1)\),
\begin{align*}
\Pr  \left(\big|\overline M_k(s) - q_b(\mu_k)\big|
\le
\sqrt{\tfrac{2 q_b(\mu_k)\left(1-q_b(\mu_k)\right) \log(2/\delta)}{m}}
 + 
\tfrac{2 \log(2/\delta)}{3m}
 \right)\\
 \ge 1-\delta.
\end{align*}
Since \(m=\lceil s/b\rceil\) we have \(1/m\le b/s\), hence
\begin{align}\label{eq:q_space}
\big|\overline M_k(s) - q_b(\mu_k)\big|
 \le
\sqrt{\frac{2b q_b(\mu_k)\left(1-q_b(\mu_k)\right) \log(2/\delta)}{s}}&
 + 
\frac{2b \log(2/\delta)}{3s}\notag\\
& \text{w.p. }\ge 1-\delta.
\end{align}

\emph{Step 3 (calibration back to \(\mu\)-space).}
Now we apply the mean–value theorem to \(g_b\) between \(q_b(\mu_k)\) and \(\overline M_k(s)\).
There exists \(\eta\) between these two points such that
\begin{align*}
\hat\mu_k - \mu_k
 = g_b(\overline M_k(s)) - g_b  \left(q_b(\mu_k) \right)
 \overset{\text{MVT}}{=} g_b'(\eta)  \left(\overline M_k(s) - q_b(\mu_k) \right).
\end{align*}
Using the inverse–function rule \(g_b'(y)=1/q_b'  \left(g_b(y) \right)\), and setting
\(\xi \coloneq g_b(\eta)\in(0,1)\) (which lies between \(\mu_k\) and \(\hat\mu_k(s)\)), we get
\begin{align}\label{eq:calib}
\big|\hat\mu_k(s)-\mu_k\big|
 = \frac{\big|\overline M_k(s) - q_b(\mu_k)\big|}{q_b'(\xi)}.
\end{align}

Finally, combining \eqref{eq:q_space} and \eqref{eq:calib} yields, with probability at least \((1-\delta)\),
\begin{align*}
\big|\hat\mu_k(s)-\mu_k\big|
 \le
\frac{1}{q_b'(\xi)} \left[
\sqrt{\tfrac{2b q_b(\mu_k)\left(1-q_b(\mu_k)\right) \log(2/\delta)}{s}}
 + 
\tfrac{2b \log(2/\delta)}{3s}
\right] \\
\qquad(\xi\in(0,1)).
\end{align*}
\end{proof}

\section{Proof of Theorem~\ref{thm:median_ucb_regret}}
\label{sec:appendix_B}
To bound the regret of Algorithm~\ref{alg:mom_ucb_v}, we first control the estimation error \(|\widetilde\mu_k-\mu_k|\) via Lemma~\ref{lem:3}, which gives a high–probability deviation bound driven by the empirical variance plus a small-sample term. Using the inequality in~\eqref{eq:lem_3}, we enforce a sufficiency condition~\eqref{eq:suff} to define a per–item empirical threshold \(\tau\) after which a suboptimal item’s index falls below that of any optimal item. By upper-bounding the number of times each suboptimal item can be selected by this threshold and plugging these counts into the cascade regret decomposition of \citet{kveton2015cascading} (Theorem~1), we obtain the final gap-dependent regret bound.

\MainLemmaThree*

\begin{proof}
By Lemma~\ref{lem:2}, we know that,
for some \(\xi_s\) between \(\mu_k\) and \(\hat\mu_k(s)\). Hence
\begin{align*}
x \le \frac{1}{q_{b}'(\xi_s)}
\Big(\sqrt{2 b q_b(\mu_k)(1-q_b(\mu_k)) a_t} + \tfrac{2 b}{3} a_t\Big).
\end{align*}

\paragraph{Variance Domination and Self-Bounding.}
For odd \(b\), the majority map pushes away from \(1/2\), so
\(q_b(p)(1-q_b(p))\le p(1-p)\) for all \(p\in[0,1]\)
(hence the Bernoulli variance cannot increase under majority).
We use the Bernoulli self-bounding inequality
\(\mu_k(1-\mu_k)\le \hat v_k(s)+x\) to obtain
\[
x \le \frac{\sqrt{2 b a_t}}{q_{b}'(\xi_s)}
\sqrt{\hat v_k(s)+x} + \frac{2 b}{3 q_{b}'(\xi_s)} a_t.
\]

\paragraph{Quadratic resolution.}
Let \(\alpha_0\coloneq \tfrac{\sqrt{2 b}}{q_{b}'(\xi_s)}\) and \(\beta_0\coloneq \tfrac{2 b}{3 q_{b}'(\xi_s)}\).
Then \(x\le \alpha_0\sqrt{(\hat v_k(s)+x) a_t}+\beta_0 a_t\).
We know that \(\sqrt{\hat v+x}\le \sqrt{\hat v}+\sqrt{x}\),

We start from
\[
x-\alpha_0\sqrt{x a_t} \le \alpha_0\sqrt{\hat v_k(s) a_t}+\beta_0 a_t.
\]
Then
\[
x-\alpha_0\sqrt{a_t x}-\bigl(\alpha_0\sqrt{\hat v_k(s) a_t}+\beta_0 a_t\bigr) \le 0.
\]
This is a quadratic inequality in \(\sqrt{x}\). Its nonnegative root yields
\[
\sqrt{x} \le \frac{\alpha_0\sqrt{a_t}+\sqrt{\alpha_0^2 a_t+4\alpha_0\sqrt{\hat v_k(s) a_t}+4\beta_0 a_t}}{2}.
\]
Squaring both sides and expanding, then applying the inequality 
\(\sqrt{u+v}\le \sqrt{u}+\tfrac{v}{2\sqrt{u}}\), which is the Taylor's expansion of the square root function and its concavity,  with 
\(u=\alpha_0^2 a_t\), \(v=4\alpha_0\sqrt{\hat v_k(s) a_t}+4\beta_0 a_t\), 
gives
\[
x  \le 2\alpha_0\sqrt{\hat v_k(s) a_t} + \bigl(2\beta_0+\alpha_0^2\bigr)a_t.
\]

hence \(x\le A_s\sqrt{\hat v_k(s) a_t}+B_s a_t\) with
\[
A_s=2\alpha_0=\frac{2\sqrt{2 b}}{q_{b}'(\xi_s)},
\quad
B_s=2\beta_0+\tfrac12\alpha_0^2=\frac{4 b}{3 q_{b}'(\xi_s)}+\frac{b}{\big(q_{b}'(\xi_s)\big)^2}.
\]
\end{proof}

Fix the horizon \(T\) and list size \(d\).
For item \(k\) at round \(t\),  Lemma~\ref{lem:3} states that,
for confidence level \(\delta\in(0,1)\),
\[
|\hat\mu_k(s)-\mu_k|
 \le 
A_s\sqrt{\tfrac{\hat v_k(s) \log(2/\delta)}{s}}
 + 
B_s\tfrac{\log(2/\delta)}{s}
 \quad \text{w.p. }\ge 1-\delta,
\]
with \(A_s=\tfrac{2\sqrt{2 b}}{q_{b}'(\xi_s)}\), \(B_s=\tfrac{4b}{3q_{b}'(\xi_s)}+\tfrac{b}{(q_{b}'(\xi_s))^2}\),
\(b=\lceil \alpha\log s\rceil\).

We set
\begin{align*}
\delta_{k, t} \coloneq  \frac{1}{(dT)^2 K \zeta(4)} t^{-4},
\qquad
\zeta(4)=\sum_{u=1}^{\infty}u^{-4}=\frac{\pi^4}{90}.
\end{align*}
Then \(\sum_{i=1}^L\sum_{t=1}^\infty \delta_{k, t}=(nK)^{-2}\).
Applying the above deviation with \(\delta=\delta_{k, t}\) yields, for each \((i,t)\),
\begin{align*}
|\hat\mu_k(s)-\mu_k|
\le
A_s\sqrt{\frac{\hat v_k(s) [ 4\log t+\Gamma ]}{s}}
 + 
B_s\frac{4\log t+\Gamma}{s},
\end{align*}
where \(\Gamma \coloneq \log \big(2(nK)^2 L\zeta(4)\big)\) is round- and item-independent and can be absorbed in the constant.

Define time-uniform coefficients
\[
A  \coloneq  \sqrt{5} \sup_{1\le s\le T} A_s,
\qquad
B  \coloneq  (4+\Gamma) \sup_{1\le s\le T} B_s,
\]
Then, for each fixed \((k,t)\),
\[
\mu_k \le \hat\mu_k(s) + A\sqrt{\frac{\hat v_k(s) \log t}{s}} + B\frac{\log t}{s}
\quad\text{w.p. }\ge 1-\delta_{k, t}.
\]

\smallskip
\noindent\emph{Union bound over all items and rounds.}
Let \(\mathcal G\) be the event that the above display holds \emph{simultaneously}
for all \(i\in[L]\) and all \(t\in\{1,\dots,T\}\).
By the union bound,
\[
\Pr(\mathcal G) \ge 1-\sum_{i=1}^L\sum_{t=1}^{T}\delta_{k, t}
 \ge 1-\sum_{i=1}^L\sum_{t=1}^{\infty}\delta_{k, t}
 = 1-(nK)^{-2}.
\]
Therefore, on \(\mathcal G\),
\begin{align*}
\mu_k \le \hat\mu_k(T_{t-1}(k))
 + \underbrace{A\sqrt{\frac{\hat v_k(T_{t-1}(k)) \log t}{T_{t-1}(k)}} + B\frac{\log t}{T_{t-1}(k)}}_{\eqcolon \rho_k(t)},\quad
\forall i, \forall t.
\end{align*}
With the radius being set to 
\begin{align*}
\rho_k(t) \coloneq 
A\sqrt{\tfrac{\hat v_k \big(T_{t-1}(k)\big) \log t}{T_{t-1}(k)}} + 
B \tfrac{\log t}{T_{t-1}(k)}.
\end{align*}
The index of item \(k\) at round \(t\) is
\[
U_t(k) \coloneq \hat\mu_k \big(T_{t-1}(k)\big) + \rho_k(t).
\]
At round \(t\), the learner recommends the \(d\) items with the largest values of \(U_t(\cdot)\).

Therefore with probability at least \(1-(nK)^{-2}\) on which, simultaneously for all items \(k\) and all rounds \(t\),
\[
\big|\hat\mu_k \big(T_{t-1}(k)\big)-\mu_k\big|
 \le \rho_k(t).
\]
In particular, assuming \(\mathcal G\),
\begin{align}\label{eq:index_ineq}
U_t(k) = \hat\mu_k \big(T_{t-1}(k)\big)+\rho_k(t)
 \le \mu_k+2 \rho_k(t),
\quad
U_t(k) \ge \mu_k.
\end{align}

Therefore a sufficient condition for any suboptimal item \(k\)'s index to be less than optimal item \(k^*\) at round \(t\) is 
\begin{align}\label{eq:suff}
\rho_k(t) \le \tfrac12 \Delta_{k, k^*}.
\end{align}

Define the threshold
\begin{align}\label{eq:tau-emplicit}
\tau_{k, k^*}(t) \coloneq\
\inf\Bigl\{s\ge 10 C: 
A\sqrt{\tfrac{\hat v_k(s) \log t}{s}}
 +  B\tfrac{\log t}{s}
 \le \tfrac12 \Delta_{k, k^*}
\Bigr\}.
\end{align}
And thus
\begin{align}\label{eq:tau-emp}
\tau_{k, k^*}(t) \coloneq\
\inf\Bigl\{  s\ge 10 c_k: 
s \ge \frac{16 A^2 \hat v_k(s) \log t}{\Delta_{k, k^*}^2}
 \text{ and }\
s \ge \frac{4 B \log t}{\Delta_{k, k^*}}
\Bigr\}.
\end{align}
Therefore, assuming that the optimal items are \(\{k^{(1)}, \dots, k^{(d)}\}\) such that \(\mu^{(1)} \ge \dots \ge k^{(d)}\), the suboptimal item \(k\) can be mistaken with the \(j\)-th optimal item at most \(m_{k, j} \coloneq \tau_{k, j} - \tau_{k, j - 1}\) times. 

From Theorem~1 in \citet{kveton2015cascading} we can write the regret as
\begin{align*}
R(T) =& \sum_{k \notin S^*} \sum_{j \in S^*} \Delta_{k, j} m_{k, j}\\
\le& \sum_{k \notin S^*} \sum_{j \in S^*} 
\begin{aligned}\Delta_{k, j} \Big[&\left(\frac{16 A^2 \hat v_k(s) \log T}{\Delta_{k, j}^2} - \frac{16 A^2 \hat v_k(s) \log T}{\Delta_{k, j-1}^2}\right)\\
+&\left(\frac{4 B \log T}{\Delta_{k, j}} - \frac{4 B \log T}{\Delta_{k, j-1}}\right)\Big]
\end{aligned}\\
\le& \sum_{k \notin S^*} \sum_{j \in S^*} 
\begin{aligned}\Delta_{k, j} \log T \Big[\mu_k(1 - \mu_k)\left(\frac{1}{\Delta_{k, j}^2} - \frac{1}{\Delta_{k, j-1}^2}\right)\\
+ \left(\frac{1}{\Delta_{k, j}} - \frac{1}{\Delta_{k, j-1}}\right)\Big]
\end{aligned}\\
\overset{(a)}{\le}& \sum_{k \notin S^*} \log T\left[\frac{\mu_k(1 - \mu_k)}{\Delta_{k}} + 1 + \log(\frac{1}{\Delta_{k}})\right]\\
\le& \sum_{k \notin S^*} \frac{\log T}{\Delta_k}
\end{align*}
In inequality~(a), we apply Lemma~3 of \citet{kveton2014matroid} to bound the first term and use a telescoping–sum argument for the second term. This completes the proof of Theorem~\ref{thm:median_ucb_regret}. \qed

\section{\CBARBARNL{}: Extending \CBARBAR{} to Cascading Bandits}
\label{sec:appendix_C}
We extend \CBARBAR{}~\cite{xu2021simple} to \CBARBARNL{} by replacing the linear reward evaluations with the cascading reward.
In particular, Lines~\ref{line:optimistic_reward}–\ref{line:pessimistic_rewards} and the selections in Lines~\ref{line:update_arm_superarms}–\ref{line:update_best_superarm} now use \(R(\cdot,\cdot)\) for cascade feedback, while the epoch schedule, probability mixing, and gap updates follow the original template.
In particular, Lines~\ref{line:optimistic_reward}–\ref{line:pessimistic_rewards} and the selections in Lines~\ref{line:update_arm_superarms}–\ref{line:update_best_superarm} 
now use the cascading reward function in~\eqref{eq:reward}, while the epoch schedule, probability mixing, and gap updates follow the original template.

\begin{proposition}
\label{prop:cbarbar}
The regret of \CBARBARNL{} (Algorithm\ref{alg:cbarbar}) is bounded by:
\begin{align}\label{eq:cbarbar}
\Reg(T) \le O\left(dC + \frac{d^2 K}{\Delta_{\min}} \log^2 T \right).
\end{align}
\end{proposition}

To prove Corollary~\ref{prop:cbarbar}, we follow the analysis of \citet{xu2021simple} and extend it to the cascading bandit setting by invoking bi-Lipschitz bounds for the cascade reward (see \cite{chen2025continuous}). In particular, for the per-round corruption aggregator we use the elementary inequalities
\[
\frac{1}{d}\sum_{k\in S(t)} \max_{j\in S(t)} |c_{j,t}|
 \le 
\max_{k\in S(t)} |c_{k,t}|
 \le 
\sum_{k\in S(t)} \max_{j\in S(t)} |c_{j,t}|,
\]
and apply these bounds in Lemmas~4–7 where the linear proof uses linearly decomposed rewards. This substitution controls the corruption terms under cascade feedback and yields the regret bound stated in~\eqref{eq:cbarbar}.

\begin{algorithm}[ht]
\caption{\CBARBAR{} with cascading reward (\CBARBARNL{})}
\label{alg:cbarbar}
\begin{algorithmic}[1]
\Input Confidence parameter \( \delta \in (0,1) \), time horizon \( T \)

\State Initialize \( \Delta_{k}^{1} \gets 1 \) and \( S_{k}^{1} \) as any valid list containing item \(k\), for all \(k \in [K]\). \label{line:init_delta}
\State Initialize \( S_{*}^{1} \) as any valid list. \label{line:init_Zstar}
\State Set \( \lambda \gets 1024\log^{2}\!\left(\frac{8K}{\delta}\log^{2} T\right) \). \label{line:init_lambda}

\For{epochs \( m = 1, 2, \dots \)} \label{line:epoch_loop}
    \State \( n_{*}^{m} \gets \lambda d^{2} K \cdot 2^{(m-1)/2} \) \Comment{Pull count for best list} \label{line:n_star_case_else}

    \For{each item \( k \in [K] \)}
        \State \( n_{k}^{m} \gets \lambda \left( \frac{\Delta_{k}^{m}}{d} \right)^{-2} \) \Comment{Pull counts from gap estimates} \label{line:compute_nim}
    \EndFor

    \State \( N^{m} \gets \sum_{k=1}^{K} n_{k}^{m} + n_{*}^{m} \) \Comment{Total pulls this epoch} \label{line:Nm_total}
    \State Set \( q_{k}^{m} \gets \frac{n_{k}^{m}}{N^{m}} \), \( q_{*}^{m} \gets \frac{n_{*}^{m}}{N^{m}} \) \Comment{Sampling probs.} \label{line:sampling_probs}

    \For{rounds \( t = T_{m-1}+1, \dots, T_{m} \)}
        \State Sample \( S_{k}^{m} \) with prob. \( q_{k}^{m} \), or \( S_{*}^{m} \) with prob. \( q_{*}^{m} \), and observe \( \widetilde X_{t,k} \) for every \( k \in S_{t} \). \label{line:sample_superarm}
    \EndFor

    \State \( \hat{\mu}_{k}^{m} \gets \frac{1}{n_{k}^{m}} \sum_{t \in E_{m}} \widetilde X_{t,k} \cdot \mathbb{I}[S_{t} = S_{k}^{m}] \) \Comment{Empirical means} \label{line:update_empirical}

    \State \( r_{*}^{m} \gets \max_{S \in \mathcal{S}}   R\!\left(S(t), \bm{X}_{t}\right) \) \Comment{Optimistic cascade reward} \label{line:optimistic_reward}

    \State \( r_{k}^{m} \gets \max_{S \in \mathcal{S}: \, k \in S}   R\!\left(S(t), \bm{X}_{t}\right) \) \Comment{Pessimistic cascade rewards} \label{line:pessimistic_rewards}

    \State \( S_{k}^{m+1} \gets \arg\max_{S \in \mathcal{S}: \, k \in S}   R\!\left(S(t), \bm{X}_{t}\right) \) \Comment{Best item–specific lists} \label{line:update_arm_superarms}

    \State \( S_{*}^{m+1} \gets \arg\max_{S \in \mathcal{S}}   R\!\left(S(t), \bm{X}_{t}\right) \) \Comment{Estimated best list} \label{line:update_best_superarm}

    \State \( \Delta_{k}^{m+1} \gets \max \!\left( 2^{-m/4},   r_{*}^{m} - r_{k}^{m},   \frac{\Delta_{k}^{m}}{2} \right) \) \Comment{Update gap estimates} \label{line:update_gap}
\EndFor

\end{algorithmic}
\end{algorithm}

\end{document}